\documentclass[runningheads]{llncs}
\usepackage{graphicx}
\usepackage{hyperref}
\usepackage{url}
\usepackage{float}
\usepackage{subfig}
\usepackage{dutchcal}
\usepackage{wrapfig}
\usepackage{amsmath}
\usepackage{amssymb}
\usepackage{amsfonts}
\usepackage{algorithm}
\usepackage{algorithmic}
\usepackage{textcomp}
\usepackage{color}
\usepackage{xcolor}
\usepackage{url}
\usepackage{multirow}

\definecolor{MyRoyalBlue}{RGB}{65,105,225}

\begin{document}
\title{Co-modeling the Sequential and Graphical Routes for Peptide Representation Learning}
\titlerunning{Co-modeling the Sequential and Graphical Routes for Peptide}
%
\author{Zihan Liu\inst{1,2,3}\orcidID{0000-0001-6224-3823} \and
Ge Wang\inst{1,2,3}\orcidID{0000-0001-8553-6493} \and \\ Jiaqi Wang\inst{3}\orcidID{0000-0002-5045-1497} \and Jiangbin Zheng\inst{1,2,3}\orcidID{0000-0003-3305-0103}  \and
Stan Z. Li\thanks{Corresponding Author}\inst{1,3}\orcidID{0000-0002-2961-8096}}
\authorrunning{Liu et al.}
%
\institute{AI Lab, Research Center for Industries of the Future, Westlake University, Hangzhou 310030, China \and
School of Computer Science And Technology, Zhejiang University, Hangzhou 310024, China \and School of Engineering, Westlake University, Hangzhou, China\\
\email{\{liuzihan,wangge,wangjiaqi,zhengjiangbin,stan.zq.li\}@westlake.edu.cn}}
\maketitle              
\begin{abstract}
Peptides are formed by the dehydration condensation of multiple amino acids. The primary structure of a peptide can be represented either as an amino acid sequence or as a molecular graph consisting of atoms and chemical bonds. Previous studies have indicated that deep learning routes specific to sequential and graphical peptide forms exhibit comparable performance on downstream tasks. Despite the fact that these models learn representations of the same modality of peptides, we find that they explain their predictions differently. Considering sequential and graphical models as two experts making inferences from different perspectives, we work on fusing ‘expert knowledge’ to enrich the learned representations for improving the discriminative performance. To achieve this, we propose a peptide co-modeling method, RepCon, which employs a contrastive learning-based framework to enhance the mutual information of representations from decoupled sequential and graphical end-to-end models. It considers representations from the sequential encoder and the graphical encoder for the same peptide sample as a positive pair and learns to enhance the consistency of representations between positive sample pairs and to repel representations between negative pairs. Empirical studies of RepCon and other co-modeling methods are conducted on open-source discriminative datasets, including aggregation propensity, retention time, antimicrobial peptide prediction, and family classification from Peptide Database. Our results demonstrate the superiority of the co-modeling approach over independent modeling, as well as the superiority of RepCon over other methods under the co-modeling framework. In addition, the attribution on RepCon further corroborates the validity of the approach at the level of model explanation. Code is available on: \href{https://github.com/Zihan-Liu-00/RepCon}{\textcolor{MyRoyalBlue}{https://github.com/Zihan-Liu-00/RepCon}}
\end{abstract}

\clearpage
\section{Introduction}

Peptides are short chains of amino acids, referred to as mini-proteins, formed through dehydration \cite{langel2009introduction}.
In recent years, peptides have gained significant attention in Artificial Intelligence (AI) research due to their immense potential in various fields such as drug discovery \cite{muttenthaler2021trends}, peptide synthesis \cite{mohapatra2020deep}, and precision medicine \cite{bhinder2021artificial}.
By harnessing the power of AI, researchers can analyze and predict peptide structures \cite{mcdonald2023benchmarking}, interactions \cite{lei2021deep,cunningham2020biophysical}, and properties \cite{batra2022machine,veltri2018deep}.
These advancements have led to groundbreaking advancements in high-throughput screening and drug design and, thus, to revolutionize healthcare.

Peptides have three levels of structure: primary, secondary, and tertiary structure.
Due to their short length, peptides typically have simpler secondary structures \cite{marullo2013peptide,seebach2006helices} and less stable tertiary structures \cite{pace1998helix,mittal2013structural}.
As a result, AI for peptides pays attention to extracting the information from the primary structure. 
Among related works, it can be summarized into three encoding approaches for representing peptides' primary structure digitally as the input of predictive models.
Firstly, a one-dimensional (1D) vector encoding represents amino acids by the concatenating one-hot encodings \cite{armenteros2019detecting,batra2022machine}.
Secondly, a sequence-based encoding considers the positional information of each amino acid in the peptide sequence \cite{charoenkwan2021bert4bitter,chu2022transformer}. 
Thirdly, a graph-based encoding represents atoms as nodes and chemical bonds as edges \cite{yan2023samppred,wei2021atse}.
It is important to note that these three ways of encoding peptide primary structures are strictly convertible, but the choice of encoding affects the construction of the AI model.
Due to the popularity of traditional machine learning methods in peptide research, 1D vector encoding is the most common encoding method.
For the remaining encodings, sequential models, such as Recurrent Neural Networks (RNNs) \cite{elman1990finding} and Transformer \cite{vaswani2017attention}, are applicable for sequence encoding, while graph encoding needs to be modeled by Graph Neural Networks (GNNs) \cite{kipf2016semi}.

A standardized empirical study on these encoding methods has been conducted in the previous work \cite{liu2023assembly}.
The researchers evaluate machine learning backbone models using the mentioned encodings and show that the top-performing models in sequential and graphical encoding groups, i.e., Transformer and GraphSAGE \cite{hamilton2017inductive}, perform comparably well and outperform 1D vector encoding.
Despite the comparable performances, numerous studies on model architectures and mechanisms lead us to explore the explanation of sequential and graphical routes in peptide modeling.
We propose an attribution method for peptides based on integrated gradients \cite{sundararajan2017axiomatic} to explain both sequence- and graph-based models and analyze the distribution of contribution scores at the amino acid level.
Our results show that the models exhibit inconsistent attribution at the amino acid level, suggesting that sequence- and graph-based models have different explanations for their predictions.
A brief overview of related work on attribution techniques and AI for peptides is provided in Appendix \ref{related_work}.

The similar performance of the two models but differences in explanation, as shown in Fig.~\ref{fig1}, suggest that neither the sequential model nor the graphical model is competent enough to capture the representations of specific amino acids from a statistical perspective.
We heuristically argue that collaboration between sequence and graph models can circumvent their respective failed predictions on certain samples.
Considering the explanation from either modeling route as expertise, we, therefore, attempt to fuse learned representations from both sequence- and graph-based models to enrich latent information for downstream tasks.
To achieve this, we propose a sequential and graphical co-modeling framework consisting of a sequential encoder, a graphical encoder, a fusion module, and a downstream predictor for learning representations of peptide primary structure.
After evaluating general fusion approaches, such as concatenation, weighted sum, and Compact Bilinear Pooling (CBP) \cite{gao2016compact}, we propose Regularization by Representation Contrasting (RepCon) and name our preferred implementation of the co-modeling framework after it.
RepCon is built upon the contrastive learning algorithm, which enhances the mutual information of sequential and graphical representations (i.e., the encoder outputs) for fusing information and as a regularization term to facilitate model training.
Compared to other fusion approaches, RepCon has the following merits.
RepCon enhances learning consistent representations for peptides across different encoders. 
The expected consistency stems intuitively from the fact that each encoder aims to extract task-related information from the same peptide.
For practicality, RepCon inactivates its graph-related part to improve inference efficiency by saving data processing and model propagation time.

In summary, the primary contributions of this paper include: 
(1) We demonstrate the differences in explanation between peptide sequence-based modeling and graph-based modeling by attribution technique. To the best of our knowledge, this is the first study to explain and then compare the explanation differences between peptide deep learning routes.
(2) We pioneer the concept of co-modeling representation learning of peptide primary structures. We explore several potential modes of co-modeling and select RepCon based on contrastive learning as the preferred option. We provide theoretical support for RepCon and detail its fair reasoning efficiency.
(3) We conduct experiments to evaluate RepCon on four benchmark datasets, including classification and regression tasks. RepCon consistently outperforms co-modeling baselines that differ from the fusion module. 
Furthermore, the improvement in the similarity between RepCon's attribution and the attribution of the original sequential and graphical backbones validates the effectiveness of our method.

\section{Attribution on Peptide}
\label{Attribution}

\begin{figure}[t] 
    \centering
    \subfloat{\includegraphics[width=0.49\linewidth]{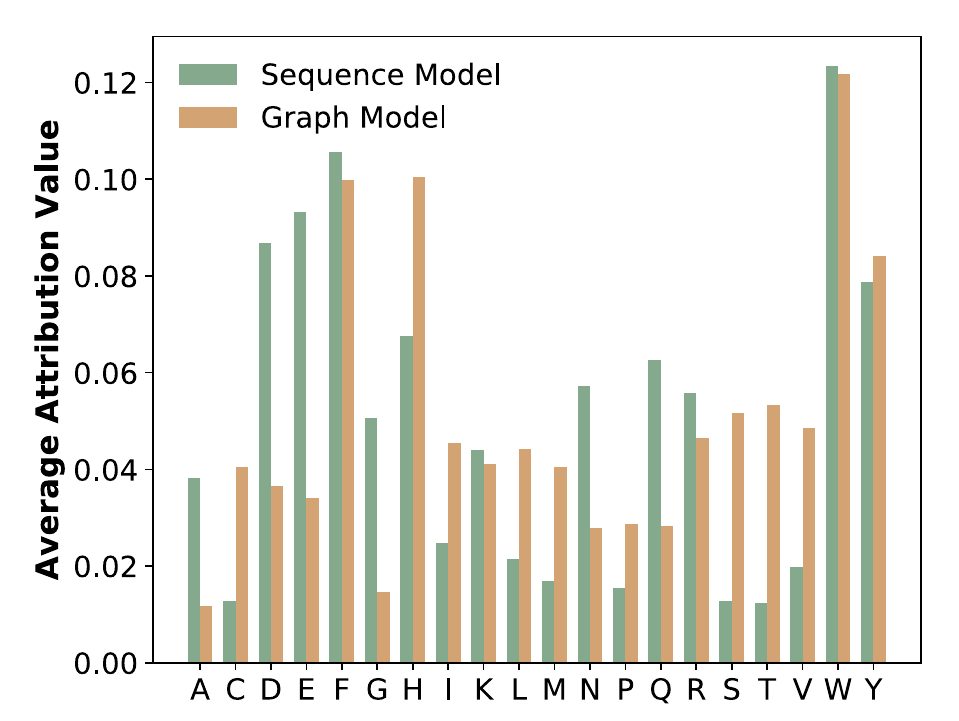}}
    \subfloat{\includegraphics[width=0.49\linewidth]{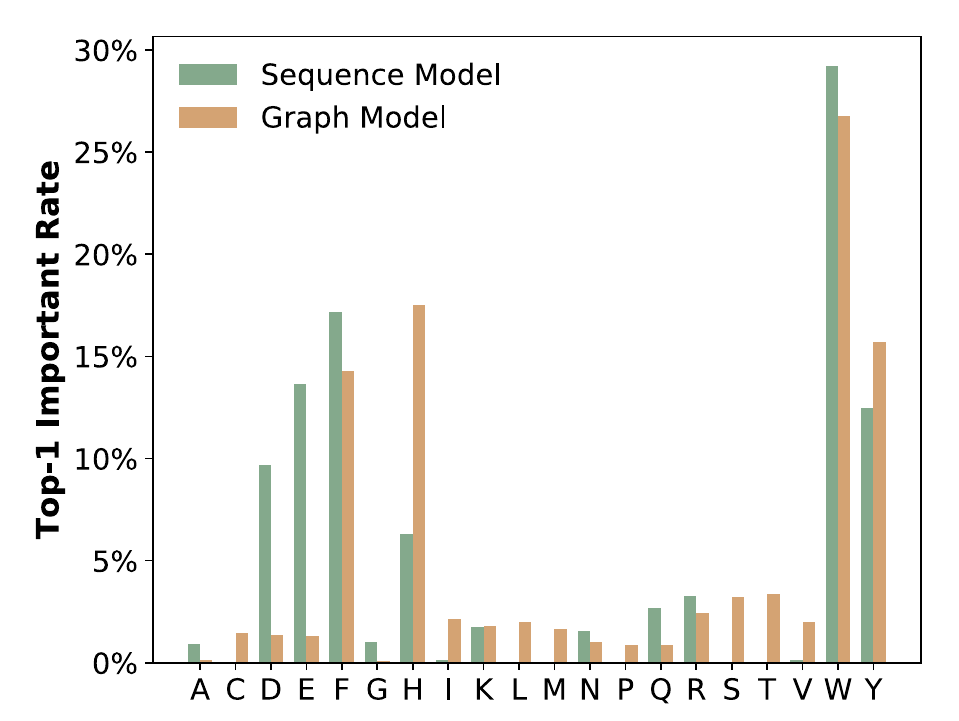}}
    \caption{The amino acid-wise attribution of the sequence model Transformer and the graph model GraphSAGE is evaluated on the test set of the AP regression task. The left histogram shows the average attribution value for each amino acid, while the right histogram displays the proportion of each amino acid identified as the most important component in the test peptides. 
    }
    \label{fig1}
\end{figure}

Previous empirical study \cite{liu2023assembly} suggests that the sequence route with the model architecture of Transformer \cite{vaswani2017attention} and the graph route with that of GraphSAGE \cite{hamilton2017inductive} achieve comparable performance on predicting aggregation propensity (AP).
\textit{However, it remains unclear whether the similar performance of these models with different mechanisms implies similar explanations for their predictions.}

In order to address this question, we develop an attribution method to elucidate the functioning of deep learning models for peptides. This method is based on the integrated gradient \cite{sundararajan2017axiomatic}. The details of attribution methods for both sequence and graph peptide models are provided in Appendix \ref{appen1_1}.
We subsequently verify the plausibility of the attribution results in Appendix \ref{appen1_2}.
We apply attribution to both the sequence and graph models for AP prediction, aiming to gain insights by examining the similarities and differences in the attribution results.
Fig.~\ref{fig1} presents the statistical results of amino acid-wise attribution for both models.
The left panel displays the average attribution values, while the right panel illustrates the proportion of each amino acid attributed as the most significant element in the peptides. 
See Appendix \ref{amono_acids} for the used abbreviations of amino acids.

Notably, distinct patterns emerge from the significant attribution differences observed among amino acids under the two different models.
The amino acids significantly attributed to the sequence model can be categorized into three subgroups. Amino acids A and G share the characteristic of having the simplest side chain, resulting in most graph nodes belonging to the backbone, which is a fundamental component of any amino acid. Amino acids D and E have negative charges. Amino acids Q and N have side chains that exhibit both polar and hydrophilic properties.
Similarly, there are commonalities among amino acids significantly attributed to the graph model. Amino acids C and M have side chains containing sulfur atoms. Amino acid C has a thiol group (-SH), while amino acid M contains a thioether group (-S-CH3). Amino acids S and T share the characteristic of having a hydroxyl group (-OH) in their side chains, making them polar and hydrophilic. On the other hand, amino acids V, L, and I consist solely of carbon chains.

In Appendix \ref{appen7}, Fig.~\ref{fig_sup1} presents a comparison of the Mean Absolute Error (MAE) at the amino acid level in the predictions of the Transformer and GraphSAGE models.
As can be observed in the figure, the differences in amino acid-level performance between the models are consistent with the differences in amino acid attribution. Building upon the aforementioned discussion, the inconsistency in attribution between the models translates into divergent explanations for their predictions, consequently impacting their predictive performance. This insight sparks the idea of integrating sequence and graph models to leverage the strengths of both approaches in representation learning, thereby providing more comprehensive information for downstream tasks.

\section{Methodology}

To further exploit the diverse explanations provided by sequence- and graph-based models, we develop a co-modeling framework that combines the information obtained from the representations learned by both models.
As illustrated in Fig.~\ref{Fig2}, the proposed co-modeling framework comprises a sequential encoder and a graphical encoder, independently extracting representations from amino acid sequences and molecular graphs, respectively. It also includes a fusion module for integrating the representations from both encoders and a downstream predictor.

While focusing on the fusion module, general methods, such as concatenation, weighted sum, and CBP, suffer from several drawbacks: 
(1) The lack of interaction between representations.
(2) The absence of additional losses beyond the monitoring signal to optimize the encoders.
(3) The training speed varies between encoder architectures, which requires careful tuning to ensure co-activation with the shared predictor.
(4) The involvement of the entire co-modeling framework during inference reduces efficiency, including the cost of graph data preparation.

\begin{figure}[t]
    \centering
    \includegraphics[width=\hsize]{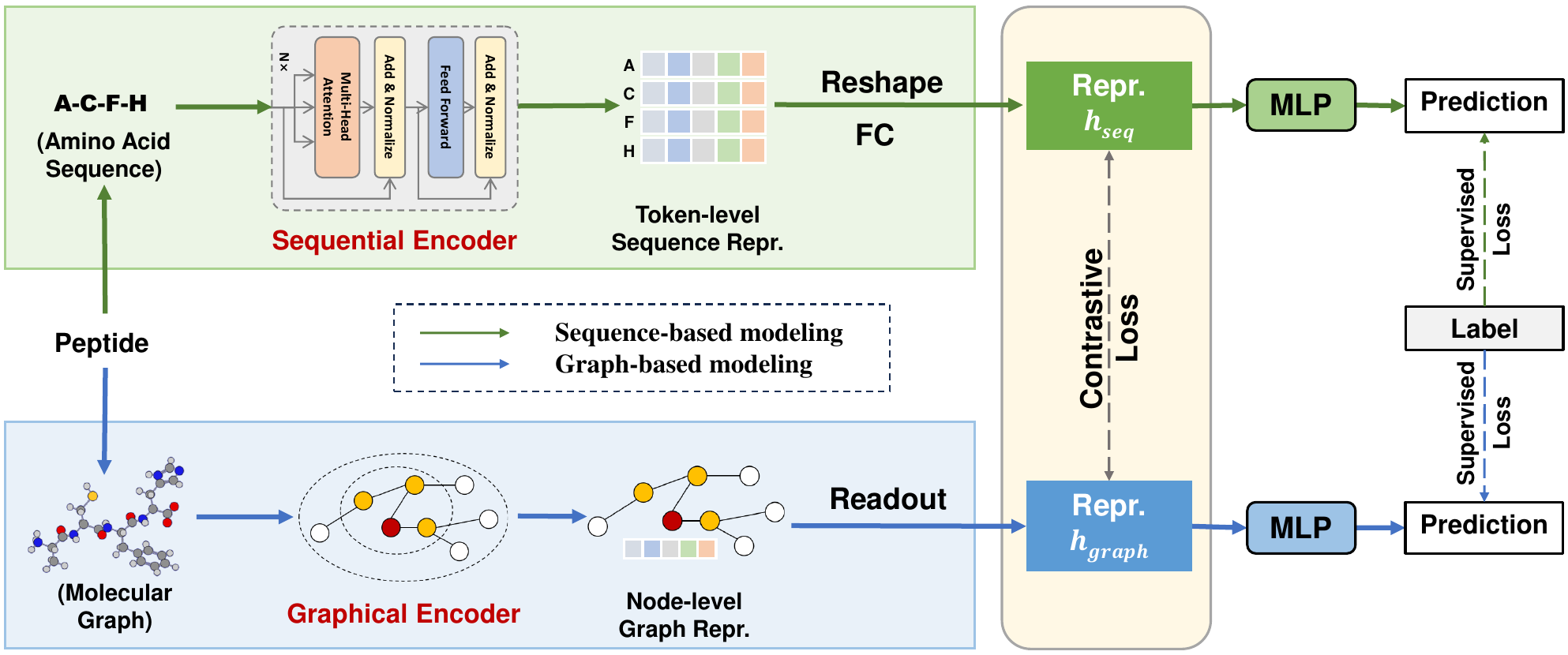}
    \caption{Illustration of the proposed RepCon in the peptide co-modeling framework during the training phase. The sequence-based and graph-based modeling routes are represented by green and blue, respectively. The consistency between the sequential and graphical encoder representations is learned through contrastive loss without supervision. Model predictions are supervised by task-relevant labels. The graph-based modeling part is deactivated during inference.}
    \label{Fig2}
\end{figure}

To address the challenges associated with fusing representations, we propose RepCon (Regularization by Representation Contrasting) as the fusion method in our co-modeling framework.
In RepCon, both the sequence encoder and the graph encoder have their own predictors (independent MLPs as shown in Figure \ref{Fig2}), which mitigates the issues arising from the differences in training speed caused by the encoder architectures.
Inspired by the contrastive learning algorithm, we incorporate the consistency of the representations from both encoders as a regularization term to facilitate information fusion.
The amino acid sequence of a peptide and its molecular graph both describe the primary structure information, and they have a bijective relationship.
Therefore, taking the representations extracted from the sequence and the graph encoding of a peptide as a positive pair, contrastive learning is employed to enhance the mutual information of the representations as an additional unsupervised signal.

Given an input peptide $(x_{seq}, x_{graph}, y)$, where $x_{seq}$ represents the amino acid sequence, $x_{graph}$ represents the molecular graph, and $y$ represents the label, we extract representations using the sequence encoder $f_{e}(\cdot)$ and the graph encoder $g_{e}(\cdot)$, denoted as $h_{seq}=f_{e}(x_{seq})$ and $h_{graph}=g_{e}(x_{graph})$, respectively.
The corresponding predictors for the encoders are denoted as $f_{p}(h_{seq})$ and $g_{p}(h_{graph})$.
Each encoder and its corresponding predictor form a complete end-to-end network.
To ensure the predictive performance for each peptide, we minimize the prediction losses associated with downstream tasks under supervision:
\begin{equation}\label{L_pred}
    \mathcal{L}_{pred} = \mathbb{E}_{(x,y)\sim D^{tr}}[\mathcal{l}(f_p(f_e(x_{seq})),y)+\mathcal{l}(g_p(g_e(x_{graph})),y)],
\end{equation}
where $D^{tr}$ represents the training data, and $\mathcal{l}(\cdot)$ represents the supervised loss function.
It is a Mean Square Error (MSE) loss for regression tasks and a cross-entropy loss for classification tasks.
To enhance the mutual information between the representations of the peptides extracted by the encoders, we employ the InfoNCE loss \cite{he2020momentum} to capture the correlation between the representations using a contrastive learning framework. The theoretical effectiveness of applying InfoNCE is demonstrated in Theorem \ref{theo1}. 

\begin{theorem} \label{theo1}
Treat the sequence data $x_{seq}$ and the graph data $x_{graph}$ as two distinct observations of a type of peptide. Fusing the learned representations $h_{seq}$ and $h_{graph}$ from $x_{seq}$ and $x_{graph}$, incorporating the InfoNCE contrastive loss, can improve the discriminative power of the model (Proof in Appendix \ref{proof_theo1}).
\end{theorem}

A positive sample pair consists of the sequence representation $h_{seq}$ and the graph representation $h_{graph}$ of a specific peptide type.
In contrast, a negative sample pair consists of representations from either the sequential or the graphical encoder but from different peptides.
For the $i$-th peptide, the formulation of its contrastive loss is as follows:
{
\begin{equation}\label{L_infonce}
    \mathcal{L}_{con}= \mathbb{E}_{(x,y)\sim D^{tr}}[-log\frac{\exp({f_e(x_{seq})^T\cdot g_e(x_{graph})}/\tau)}{\sum_{j=0}^K \exp({f_e(x_{seq})^T\cdot h^-}/\tau)+\exp({g_e(x_{graph})^T\cdot h^-}/\tau)}],
\end{equation}
}
where $K$ is the number of negative samples, $h^-\in \{f_e(x_{seq,j}),g_e(x_{gra,j})\}$ denotes a representation from either the sequential or graphical encoder from another peptide, and $\tau$ represents a temperature hyper-parameter \cite{wu2018unsupervised}.
The improvement from the mutual information perspective by minimizing $\mathcal{L}_{con}$ is theoretically analyzed in Appendix \ref{proof_theo2}.
Since the sequential and graphical representations are essentially extracted from the same peptide, the encoders have consistent optimization goals in extracting task-relevant information.
In addition, the loss targeting consistency also serves as an unsupervised regularization term, preventing overfitting, improving generalization, and controlling model complexity.

To incorporate the contrastive loss into the training process, we add it to the supervised loss term in Eq.~\ref{L_pred} and obtain the final loss as follows:
\begin{equation} \label{L_train}
\mathcal{L}_{train}=\mathcal{L}_{pred}+\lambda \mathcal{L}_{con},
\end{equation}
where $\lambda$ is a hyper-parameter to balance these two loss terms.
In the inference phase, the model activates only the sequential encoder and its predictor, resulting in the following prediction:
\begin{equation} \label{inference}
    y'=f_p(f_e(x_{seq})).
\end{equation}
We activate the sequential end-to-end network instead of the graphical part for the following reasons:
On the one hand, the sequential encoder has a larger number of parameters, performs better, and empirically exhibits more stable training.
On the other hand, the sequential encoder directly processes the raw FASTA format.
The pseudocodes for the training and inference phases of RepCon are detailed in Alg. \ref{alg1} in Appendix \ref{appen3}.

\section{Experiments} \label{exp}

In this section, we conduct experiments to validate and understand the effectiveness of RepCon. We describe the experimental scenarios, including the test datasets, baselines, and experimental setup. Next, we present RepCon's performance on classification and regression datasets, computational efficiency, and attribution. In particular, we highlight experiments or discussions related to the following questions:
\textbf{Q1}: How to implement RepCon, including the rationale behind the choice of Transformer and GraphSAGE as backbones and the selection of hyperparameters?
\textbf{Q2}: How to implement baseline fusion methods for co-modeling?
\textbf{Q3}: What experiments demonstrate the performance of RepCon, both during the training and testing phases?
\textbf{Q4}: What insights can be gained from comparing RepCon with its backbones from an explanatory perspective?

\subsection{Experimental Scenarios}

\noindent \textbf{Datasets} The regression task dataset includes the predictions of aggregation propensity (AP) \cite{liu2023assembly} and retention time (RT) \cite{wen2020cancer}.
The classification task dataset includes the prediction of antimicrobial peptides (AMP) \cite{bhadra2018ampep} and the peptide family of PeptideDB (PepDB) \cite{liu2008construction}.
To ensure consistency, we further filtered the samples to include only 20 natural amino acids and a maximum peptide length of 50.
Statistics of the datasets are provided in Appendix \ref{appen4}.
The dataset AP has already been split into train/validation/test sets in the previous work \cite{liu2023assembly}.
For the dataset RT, \cite{wen2020cancer} have already split it into train and test sets. We further split the train set into train and validation sets at a ratio of 9:1.
The datasets PepDB and AMP were not provided with any predefined splits. Therefore, we divide them into train, validation, and test sets at a ratio of 8:1:1.

\noindent \textbf{Baselines} 
We have chosen the models presented in \cite{liu2023assembly} as our baselines for benchmarking. This decision was made because methods from other dataset resources incorporate external biological information, which would not allow for fair comparisons.
Our main focus is to compare the fusion performance of RepCon with other co-modeling methods utilizing Transformer and GraphSAGE encoders.
The co-modeling baselines we consider include Weighted Sum (WS), Concatenation (Concat), Cross Attention (CA), and Compact Bilinear Pooling (CBP). Additional details regarding these co-modeling baselines can be found in Appendix \ref{appen5} (Q2).
Among the sequence-based and graph-based baselines, Transformer and GraphSAGE have demonstrated superior performance. Therefore, we have chosen them as the backbone encoders for our co-modeling methods. The evaluations supporting this choice are provided in Appendix \ref{appen6} (Q1).

\noindent \textbf{Experimental Setup (Q1, Q2)} 
The models' hidden layer dimensions are set to 64. 
Both the predictor part in RepCon and the baseline/backbone models consist of fully connected layers and LeakyReLU activation.
The selection of the hyperparameter $\lambda$, which is used to balance the two loss terms of RepCon, is described in detail in Appendix \ref{appen9}.
The experiments are conducted in a running environment with CUDA 11.7, Python 3.7.4, and PyTorch 1.13.1. 
The graph peptides are coarse-grained following Martini 2 \cite{monticelli2008martini}. The implementation has been described in \cite{liu2023assembly}.
The graph convolutional layers are implemented via DGL 1.1.1 \cite{wang2019dgl}.
The metrics used to evaluate the performance of the benchmark model in the regression task are MAE, MSE, and R-squared (R$^2$). For the classification task, the metric used is Accuracy.

\subsection{Effectiveness on Regression Tasks (Q4)}

\begin{table}[t]
\centering
\caption{Experimental results of RepCon and the baseline co-modeling methods are presented for the regression tasks on datasets AP and RT. The evaluation metrics used for the regression task include MSE, MAE, and R$^2$. The best-performing method is indicated in bold. The necessary components of the model during the inference phase are checked.}
\scalebox{0.93}{
\begin{tabular}{llcccccccc}
\hline
\multirow{2}{*}{\textbf{Framework}} & \multirow{2}{*}{\textbf{Method}} & \multicolumn{3}{c}{\textbf{AP}} & \multicolumn{3}{c}{\textbf{RT}} & \multicolumn{2}{c}{\textbf{Inference}} \\ \cline{3-10}
                      &                         & \textbf{MAE}    & \textbf{MSE}    & \textbf{R$^2$}   & \textbf{MAE}    & \textbf{MSE}    & \textbf{R$^2$} & \textbf{Seq} & \textbf{Graph}  \\ \hline
Backbone              & Transformer             & 3.81E-2       & 2.33E-3       & 0.947    & 1.57       & 5.02       & 0.991    & $\checkmark$ & \\
Backbone              & GraphSAGE              & 3.89E-2       & 2.44E-3       & 0.945       & 2.57       & 12.9       & 0.977    & & $\checkmark$ \\ \hline
Co-modeling           & WS            & 4.05E-2       & 2.68E-3       & 0.940     & 1.92       & 7.75       & 0.986    & $\checkmark$ & $\checkmark$ \\
Co-modeling           & Concat                  & 3.75E-2       & 2.27E-3       & 0.949     & 1.45       & 4.67       & 0.992    & $\checkmark$ & $\checkmark$  \\
Co-modeling           & CA         & 3.79E-2       & 2.32E-3       & 0.948     & 1.44       & 4.52       & 0.992    & $\checkmark$ & $\checkmark$ \\
Co-modeling           & CBP & 3.76E-2       & 2.31E-3       & 0.948     & 1.48       & 4.82       & 0.992    & $\checkmark$ & $\checkmark$ \\
Co-modeling           & RepCon                 & \textbf{3.62E-2}       & \textbf{2.12E-3}       & \textbf{0.953}     & \textbf{1.41}       & \textbf{4.40}       & \textbf{0.993}    & $\checkmark$ &  \\ \hline
\end{tabular}}
\label{tab1}
\end{table}

In this section, we evaluate each method on the regression datasets. The experimental results are presented in Table~\ref{tab1}.
Each model in the co-modeling framework contains a Transformer-based sequential encoder, a GraphSAGE-based graphical encoder, an MLP predictor, and their respective fusion methods.
The backbones, Transformer and GraphSAGE, have competitive performance on the dataset AP, but Transformer outperforms GraphSAGE on the dataset RT.
The difference in performance on the dataset RT comes down to the fact that graph convolution has fewer learnable parameters and its difficulty in capturing long-term relationships between amino acids.

In the co-modeling group, WS is the only method that does not outperform backbones.
Among other methods, RepCon outperforms baselines and backbones by a wide margin.
Concat, as a parameter-free method, ranks second after RepCon.
CA and CBP, as methods with interactions in their representations and containing learnable parameters, fail to show the expected performance.
This could be due to more parameters leading to difficulties in training with insufficient samples.
Among the methods in the co-modeling framework, RepCon introduces additional unsupervised signals to further optimize the encoder, whereas the other methods only work to provide richer information to the predictor, which explains the reason for RepCon's superior performance.

The activated components, i.e., sequence- and graph-based parts, during inference, are shown in Table~\ref{tab1}. In the proposed co-modeling framework, the main factors affecting computational efficiency include the following aspects.
Unlike other co-modeling methods, RepCon activates the sequence encoder module only during the inference phase.
This saves time in preparing peptides in molecular graphs and improves propagation speed by downsizing the model.
Thus, RepCon's decoupled design makes it more efficient.

\subsection{Effectiveness on Classification Tasks (Q3)}

\begin{figure}[t]
    \centering
    \includegraphics[width=\hsize]{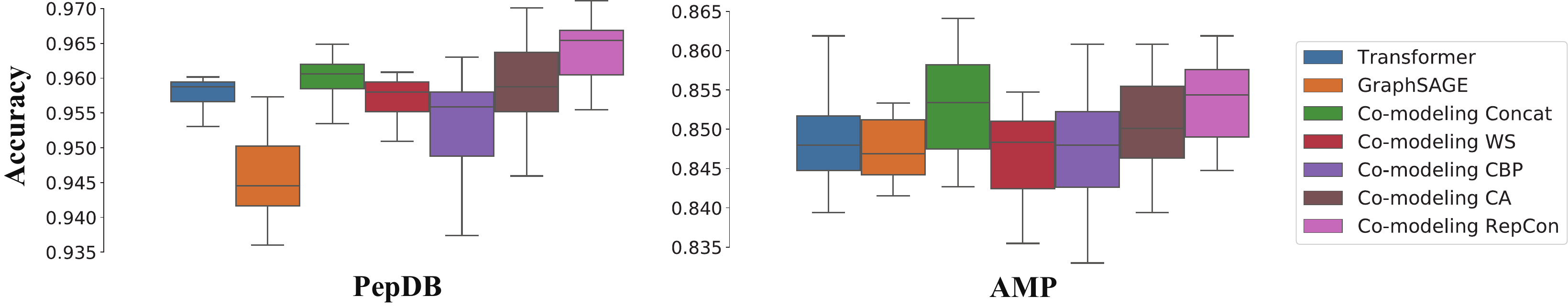}
    \caption{Experimental results on the accuracy of RepCon and baseline co-modeling methods tested on classification tasks, datasets PepDB and AMP, shown in boxplot. The boxes show the medians and the upper and lower quartiles of the results. The lines outside the boxes indicate the maximum and minimum values of the result distributions.}
    \label{Fig3}
\end{figure}

In this section, we compare the performance of co-modeling methods with backbones on the classification datasets PepDB and AMP.
We perform repeated experiments for each method and statistically tabulate the results in box plots, as displayed in Fig.~\ref{Fig3}.
Backbones perform competitively on the AMP dataset, while Transformer performs better than GraphSAGE on the dataset PepDB.
Comparing the median result among the co-modeling methods, we observe that RepCon outperforms the other methods on the PepDB dataset and slightly outperforms the second-best method, Concat, on the AMP dataset.
Among the other methods, both Concat and CA have reasonable classification accuracy, while WS and CBP perform relatively poorly.

Compared to the performance of models in classification tasks, RepCon, CA, and Concat show steady performance gains based on backbones.
WS does not demonstrate effective aggregation of information from the results due to the unlearnable merging of dimensions.
CBP, despite its solid performance on the regression tasks, does not show consistent performance gains on the classification task.
In addition, if we compare the worst case from a model's experiments (indicated by the horizontal line below), RepCon still outperforms other methods and backbones.
This means that RepCon's performance may be less affected by the randomness in initialization and training.
Of the other methods, CBP and CA have more erratic performance due to the parameter complexity of their modules.
To summarize the experimental results on classification tasks, RepCon is the best-performing model among all methods under co-modeling frameworks.

\subsection{Analysis on Attribution of RepCon (Q4)}

\begin{table}[t]
\centering
\setlength\tabcolsep{5pt}
\caption{Metrics of attribution similarities between models, where T, G, and R stand for Transformer, GraphSAGE, and RepCon. Top-$i$ overlap represents the probability that the $i$ amino acids with the top attribution values have overlap. Kendall's Tau and Spearman's Footrule evaluate the correlation between rankings; JS Divergence and Cosine Similarity evaluate the similarity between value distributions. An upward arrow indicates a higher correlation as the value increases and vice versa. Results are expressed as mean $\pm$ variance.}
\label{tab2}
\begin{tabular}{llccc}
\hline
\textbf{Evaluating}         & \textbf{Metric}               & \textbf{T vs. G}          & \textbf{G vs. R}          & \textbf{T vs. R}          \\ \hline
Order        & Kendall’s Tau↑       & 0.165$\pm$0.301 & 0.183$\pm$0.313 & 0.795$\pm$0.154 \\
Order        & Spearman's Footrule↑ & 0.227$\pm$0.379 & 0.258$\pm$0.340 & 0.886$\pm$0.115 \\
Order        & Top-1 Overlap↑       & 43.5\%       & 46.9\%       & 87.2\%       \\
Order        & Top-2 Overlap↑       & 81.5\%       & 85.3\%       & 99.7\%       \\
Distribution & JS Divergence↓       & 0.366$\pm$0.107 & 0.082$\pm$0.050 & 0.011$\pm$0.009 \\
Distribution & Cosine Similarity↑   & 0.782$\pm$0.140 & 0.817$\pm$0.125 & 0.984$\pm$0.014 \\ \hline
\end{tabular}
\label{Tab2}
\end{table}

In this section, we evaluate how RepCon fuses the knowledge of sequential and graphical encoders by measuring the attribution relevance of RepCon and backbones.
Stronger attribution relevance means that the models explain the predictions more closely.
We construct the metrics in terms of the order of importance and the distribution of the attribution values across amino acids in each peptide.
The former includes the correlation algorithms Kendall's Tau \cite{kendall1938new} and Spearman's Footrule \cite{spearman1904general}.
The latter includes Jensen-Shannon (JS) divergence \cite{lin1991divergence} and Cos similarity.
Furthermore, we introduce the more intuitive Top-$i$ Overlap.
Taking a short peptide sequence FLER as an example, if the attribution shows that model A considers F(1) and E(3) are two important amino acids, while model B considers F(1) and R(4) are two important amino acids, then we define the attributions are Top-2 Overlapped as they share an F(1).

In Table~\ref{Tab2}, we show the comparison of the attribution differences between the models under each of the metrics.
The involved models are Transformer (T), GraphSAGE (G), and RepCon (R).
It is worth noting that T and G are not trained under RepCon's framework but are separately trained backbones.
The \textbf{T vs. G} group is the control group, which shows the similarity in attribution results between backbone models.
In both terms of order and distribution, the results from \textbf{G vs. R} group illustrate that RepCon and GraphSAGE are closer in their attributions compared to T vs. G. 
This demonstrates that fusing representations in the way of RepCon's mutual information optimization allows the explanations of RepCon to be closer to how graphical models work.
Comparing the group \textbf{T vs. R} and \textbf{T vs. G}, two important phenomena can be observed.
One is that Transformer and RepCon have higher similarity relative to GraphSAGE, which is interpreted as a result of the architectural consistency of the models.
On the other hand, despite the architectural consistency, we are able to observe significant dissimilarities between Transformer and RepCon on the metrics Kendall's Tau, Spearman's Footrule and Top-1 Overlap.
We interpret these discrepancies as the model further extracting and refining the representations learned by sequences and graph encoders during the contrastive learning algorithm.
Furthermore, RepCon's superior performance relative to the backbones supports RepCon's explanations that its predictions are closer to the laws of nature inherent in the downstreamtask compared with either backbone.
In Appendix \ref{appen7}, we provide amino acid-wise attribution and MAE to demonstrate that RepCon's tendency to shift the importance of each amino acid is beneficial for discrimination rather than a trade-off.

\section{Conclusion}
In this paper, we employ an integrated gradient-based attribution method for peptides and discover the inconsistency of task-related explanations between the sequential and graphical routes of machine learning.
Considering both routes as distinct sources of expert knowledge, we propose a co-modeling framework to improve the extraction of peptides' primary structure information. We focus on designing effective fusion methods to combine representations from sequential and graphical encoders. Among all proposed methods, RepCon outperforms others in terms of both theoretical evidence and inference efficiency and is our preferred fusing implementation for the co-modeling framework.
Through empirical studies, we further demonstrate RepCon's superior performance on various discriminative tasks, validating its effectiveness in representation learning. 
Additionally, we conduct attribution evaluations on RepCon, further supporting our theoretical expectations for RepCon.

\clearpage
\appendix

\section{Related Work} \label{related_work}

\textbf{AI for Peptide} has facilitated the prediction of the peptide's properties, interactions, and discovery.
Early methods \cite{hawkins2006detecting,zhang2007prediction} typically embedded the peptide into a 1D vector, where every 20 dimensions represent one amino acid. 
This encoding is usually suitable for traditional models such as SVM \cite{hearst1998support}, RF \cite{breiman2001random}, etc.
Due to the effectiveness and interpretability of this technical route, it is still widely used in recent research \cite{batra2022machine,mongia2023adenpredictor}.
With the success of deep learning in natural language processing and graph data mining, peptides are encoded into sequences and graphs to match new AI techniques. 
Peptides are encoded into feature sequences after one-hot, Position-Specific Scoring Matrix (PSSM) or word embedding and modeled using LSTM, TextCNN, Transformer, etc. \cite{muller2018recurrent,cheng2021bertmhc,chu2022transformer}.
Wei et al. \cite{wei2021atse} propose extracting the PSSM and chemical information with CNN+LSTM and GNN and utilizing them in prediction.
Our work differs from the previous studies in that they aim at fusing multimodal information, whereas ours focuses on combining various routes of the primary structure of peptides.

\textbf{Attribution} aims to explain the inference between inputs and outputs of nonlinear black-box deep learning models \cite{linardatos2020explainable}.
It derives from the concept of
saliency \cite{itti1998model}, which depicts visually informative pixels in an image.
Simonyan et al. \cite{simonyan2013deep} proposed the first gradient-based attribution method, where each gradient quantifies how much a change in each input dimension affects the prediction of the neighborhood around the input.
In this paper, we introduce the integrated gradients \cite{sundararajan2017axiomatic} as the basic algorithm to evaluate the attribution of the sequence and graph models.
This approach measures the importance of the input dimensions directly based on the models' parameters as well as the output.
Integrated gradients have also been validated and applied on sequential and graphical neural networks \cite{sikdar2021integrated,wallace2019allennlp,sanchez2020evaluating,mccloskey2019using}.
Given the generality and validity, we adopt the integrated gradient as a metric to analyze the variability of sequential and graphical neural networks used to learn peptide representations.

\section{Attribution on Peptides} \label{appen1}

\subsection{Implementation of Attribution Method} \label{appen1_1}

Integrated gradient \cite{sundararajan2017axiomatic} has been widely recognized as an effective and reliable attribution methods \cite{arrieta2020explainable}. 
It is initially proposed for the attribution of image classification tasks, where the RGB distribution of image pixels is in continuous space. 
Unlike images, the input dimension of either sequence- or graph-based encoding is in a discrete space. 
In our implementation, we derive gradient saliency on the continuous feature space after the word embedding.

In the following, we present the steps of attribution using the sequential model, and subsequently add the differences of the graphical model.
For a seqnential model $f(X)$, where $X=\{x_1,x_2,...,x_n\}$ denotes a peptide consists of $n$ amino acids.
The feature after word embedding layer is denoted as $H=Emb(X)$, where the dimension of $H$ is $n \times d$.
For a discrete data, we aim to derive the gradient saliency from the loss to the embedding feature $H$.
The loss function is related to the downstream task. As an example, the loss function expressions used to derive gradient significance for classification and regression problems are respectively:
\begin{equation}
    \mathcal{L}_{cla}(X) = \mathcal{L}_{CE}(f(X),\hat{y}),
\end{equation}
\begin{equation}
    \mathcal{L}_{reg}(X) = |f(X)|,
\end{equation}
where $\mathcal{L}_{cla}$ for classification targets on preventing the model from predicting for the original class $\hat{y}$, and $\mathcal{L}_{reg}$ for regression simply aims to change the previous prediction.
Regardless of the downstream tasks, the purpose of the loss function is to find the dimensions that are most effective at corrupting the original predictions, which are deemed important to the model's predictions.

After defining the loss function, we derive the gradient saliency on the embedding layer $H$ by backpropagation and transform the gradient on $H$ into the attribution values of amino acids at each position. 
The formula for the integrated gradients is as follows:
\begin{equation} \label{e3}
    InteGrads(H)=H \circ \sum^{m}_{k=1} \frac{\partial \mathcal{L}(\frac{k}{m}H)}{\partial H},
\end{equation}
where $\circ$ denotes the element-wise multiplication, and $m$ is the number of steps in the Riemman approximation of the integral \cite{sundararajan2017axiomatic}.
From the gradient saliency $InteGrads(H)$ for the embedding $H$ of the peptide $X$, the last step is to sum the saliency of continuous features as the attribution of amino acids in discrete space.
For each amino acid $x_i$ in the peptide sequence $X$, the attribution for $x_i$ is denoted as:
\begin{equation} \label{e4}
Attribution(x_i)=\frac{\sum^{d}_{q=1} InteGrads(H)_{i,q}}{\sum^{n}_{p=1} \sum^{d}_{q=1} InteGrads(H)_{p,q}}.
\end{equation}
The attribution values of amino acids are normalized to ensure that the sum of all the attribution values in a peptide is equal to one.

For graph models, the input graph consists of nodes and edges, where nodes represent a bag of atoms and edges represent chemical bonds.
The species of nodes is a discrete variable, so the input $X$ needs to be embedded to the continuous space $H$ as well, in the same way as in the sequence model.
The function of the graph model is $f(X,A)$, where $A$ represents the adjacency matrix of chemical bonds.
The steps for computing the integrated gradient via the loss function are the same.
Note that for the graph model, the weight term $\frac{k}{m}$ in Eq.~\ref{e3} acts on $H$ and not on $A$.
The element $x_i$ Eq.~\ref{e4} represents atoms rather than amino acids.
As each amino acid is composed of several atoms in the molecular graph, we sum the attribution values of the atoms contained in each amino acid as the attribution value of the amino acid and likewise normalize the attribution values of the amino acids in a peptide.

\subsection{The Plausibility of Our Peptide Attribution Method} \label{appen1_2}

To validate the plausibility of our attribution method, we assess the attribution results on the peptide aggregation propensity (AP) dataset \cite{liu2023assembly}, which involves a regression task that utilizes data obtained from molecular dynamics simulations based on the Martini force field \cite{monticelli2008martini}. 
It describes the tendency of peptides to aggregate or not to aggregate due to differences in primary structures.
Simulated data offer several advantages over experimental data, including reduced noise, reproducibility, and a clear understanding of label acquisition principles \cite{nikolenko2021synthetic}. 
As a result, the labels in the AP regression dataset are reliable and explainable, aligning with our requirements for evaluating the attribution results.

Fig.~\ref{fig1} illustrates the statistical analysis of amino acid-wise attribution on the test set samples.
The left histogram presents quantitative statistics on the average attribution values for each amino acid, while the right histogram presents the distribution of the most significant amino acid type in the test set samples.
Based on the findings depicted in Fig.~\ref{fig1}, it is evident that certain amino acids exhibit substantial attribution in both the sequence and graph models.
For example, amino acids F, W, and Y are all aromatic compounds, leading to peptide aggregation through $\pi$-$\pi$ stacking.
On the contrary, amino acids K and R are positively charged, which can incur in attractive or repulsive forces between peptides and, thus, non-aggregation.
It suggests that the relevant amino acids have higher attribution values, and peptides containing these amino acids demonstrate high or low AP values of a peptide when they get a significant attribution value.
We present visualizations of different levels of aggregation attributions of a high AP and a low AP peptide in Appendix \ref{appen2} to facilitate the understanding of this analysis.
Our peptide attribution method provides reliable feedback on the selection of critical components responsible for peptide aggregation.

\section{20 Amino Acid Abbreviations and Their Physicochemical Properties} \label{amono_acids}

\begin{table}[h]
\caption{20 Amino Acid Abbreviations and Their Physicochemical Properties.}
\centering
\begin{tabular}{lllll}
\hline
Name            & Abbr & Abbr & Molecular & Residue    \\ \hline
Alanine         & Ala  & A    & 89.1      & C3H7NO2    \\
Arginine        & Arg  & R    & 174.2     & C6H14N4O2  \\
Asparagine      & Asn  & N    & 132.12    & C4H8N2O3   \\
Aspartic acid   & Asp  & D    & 133.11    & C4H7NO4    \\
Cysteine        & Cys  & C    & 121.16    & C3H7NO2S   \\
Glutamic   acid & Glu  & E    & 147.13    & C5H9NO4    \\
Glutamine       & Gln  & Q    & 146.15    & C5H10N2O3  \\
Glycine         & Gly  & G    & 75.07     & C2H5NO2    \\
Histidine       & His  & H    & 155.16    & C6H9N3O2   \\
Isoleucine      & Ile  & I    & 131.18    & C6H13NO2   \\
Leucine         & Leu  & L    & 131.18    & C6H13NO2   \\
Lysine          & Lys  & K    & 146.19    & C6H14N2O2  \\
Methionine      & Met  & M    & 149.21    & C5H11NO2S  \\
Phenylalanine   & Phe  & F    & 165.19    & C9H11NO2   \\
Proline         & Pro  & P    & 115.13    & C5H9NO2    \\
Serine          & Ser  & S    & 105.09    & C3H7NO3    \\
Threonine       & Thr  & T    & 119.12    & C4H9NO3    \\
Tryptophan      & Trp  & W    & 204.23    & C11H12N2O2 \\
Tyrosine        & Tyr  & Y    & 181.19    & C9H11NO3   \\
Valine          & Val  & V    & 117.15    & C5H11NO2   \\ \hline
\end{tabular}
\label{abbr}
\end{table}

Table \ref{abbr} lists full names, abbreviations, molecular weights, and residue composition of the 20 amino acids involved.

\section{Visualisation with Attribution of an Aggregated/Non-aggregated Peptide Pair Example} \label{appen2}

\begin{figure}[h]
    \centering
    \includegraphics[width=\hsize]{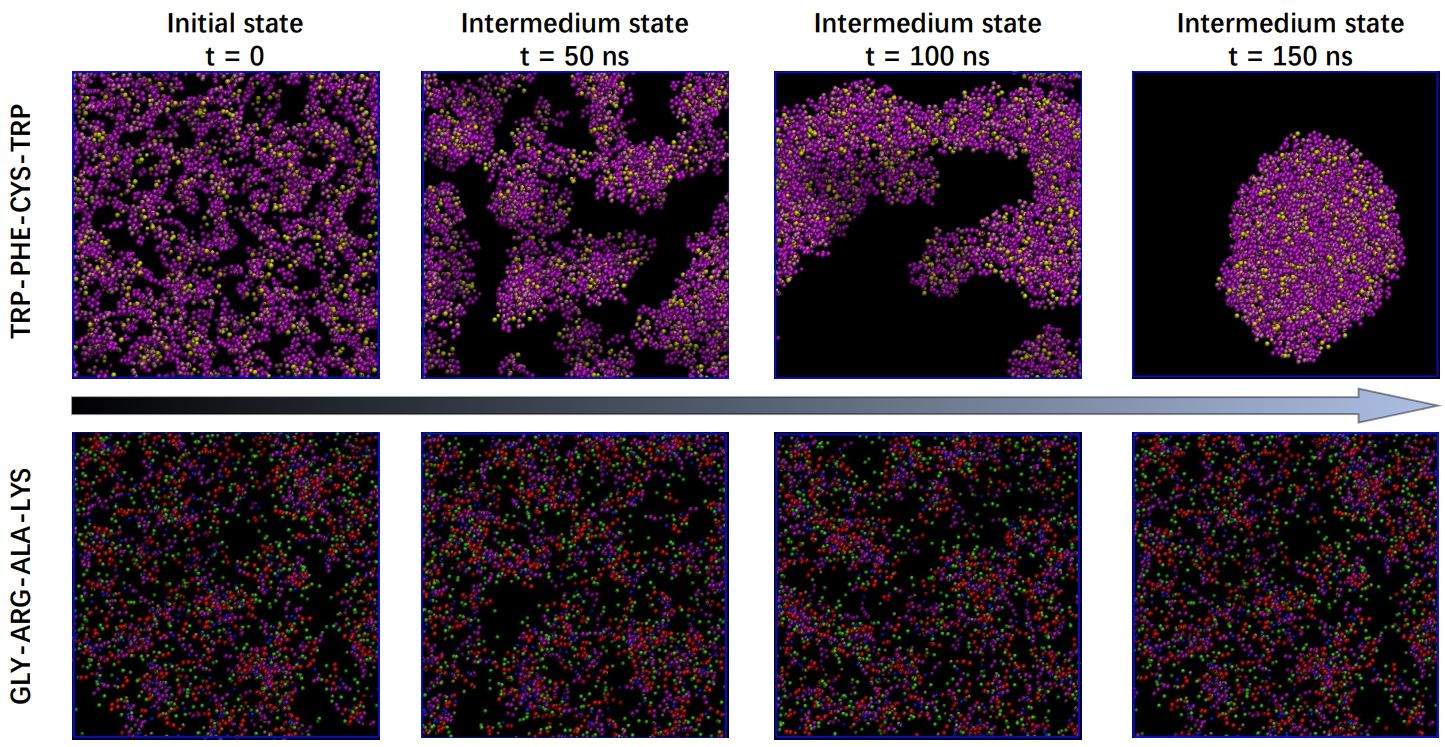}
    \caption{Visualization of the simulation process for an aggregating peptide TRP-PHE-CYS-TRP (WFCW) and a non-aggregating peptide GLY-ARG-ALA-LYS (GRAK) from the previous work \cite{liu2023assembly}. The peptide molecules show a regular motion under the force field, thus macroscopically behaving as aggregated or unaggregated.}
    \label{fig_sup2}
\end{figure}

Taking a pair of peptides with contrasting properties as an example, we conduct attribution on a peptide pair WFCW and GRAK. WFCW exhibits a high AP value, while GRAK has a low AP value. The simulation processes of these two peptides are visualized in Fig.~\ref{fig_sup2}.

Fig.~\ref{fig_sup2} shows the simulated aggregation process of peptides in an aqueous solution, which is exactly the data collection process for the dataset AP. 
The water molecules are covered by black backgrounds and the colored particles are the components of the peptides. The four figures against the top show screenshots of a sample aggregated peptide WFCW between 0 and 150 ns; the four figures against the bottom show screenshots of a sample unaggregated peptide GRAK between 0 and 150 ns.

The attribution results obtained from the sequence model, Transformer \cite{vaswani2017attention}, for the peptide WFCW are as follows: [W:0.28, F:0.28, C:0.07, W:0.37]. This implies that all amino acids except C play a crucial role in the aggregation of the WFCW peptide.
According to expert knowledge \cite{batra2022machine,wang2023transformer}, amino acids W and F are aromatic compounds. Consequently, $\pi$-$\pi$ stacking interactions occur between WFCW molecules, leading to aggregation.
In the case of GRAK, the attribution result is as follows: [G:0.17, R:0.29, A:0.25, K:0.29]. Amino acid K and R are positively charged amino acids, while all other amino acids are electrically neutral. As a result, GRAK molecules repel each other due to electrostatic interactions.

\section{Proof of Theorem 1} \label{proof_theo1}

\begin{proof}
Assume $D_c$ is a distribution of peptides with label $c$ on a given downstream task from which sequence data $x_{seq}$ and graph data $x_{graph}$ with label $c$ can be drawn,
$$
x_{seq} \sim  D_c(X|\tau=sequence);\;\;
x_{graph} \sim  D_c(X|\tau=graph)
$$
$x_{seq}$ and $x_{graph}$ are two observation forms about a certain peptide $X \sim D_c$, where the condition variable $\tau$ serves as a determining factor to classify whether the data is observed as sequence data or graph data. In our method, we consider $x_{seq}$ and $x_{graph}$ above as a positive pair $(x, x^+)$ and we denote $(h, h^+)$ as the corresponding positive pair of learned representations from encoders. We demonstrate the above definition under a classification downstream task pattern. It can be generalized to regression downstream tasks when we treat each regression value as an unique soft label of binary classification. 

In following proof, we aim to construct a function which can describe the discriminative ability of the model, and we will show such function can be bounded by contrastive loss. Specifically, minimizing this function will let model separate representations with different downstream task labels.

Thus we can reformulate and simplify the contrastive loss  $\mathcal{L}_{con}$ (Eq.~\ref{L_infonce})as following,

$$
\begin{aligned}
\mathcal{L}_{con} &= \mathop{\mathbb{E}}\limits_{\tiny{c, \{c_i^-\}_{i=1}^{k} \atop \sim C}}
\mathop{\mathbb{E}}\limits_{\tiny{(x, x^+) \sim D_{c} \atop x_i^- \sim D_{c_i^-}}}
\left[-\mathop{log}\left(\frac{e^{\left(\tiny{h^Th^+}\right)}}{e^{\left(\tiny{h^Th^+}\right)}+\sum_{i=1}^{k}e^{\left(\tiny{h^Th_i^-}\right)}}\right)\right] \\
&=\mathop{\mathbb{E}}\limits_{\tiny{c, \{c_i^-\}_{i=1}^{k} \atop \sim C}} \left[ \mathop{\mathbb{E}}\limits_{\tiny{(x, x^+) \sim D_{c} \atop x_i^- \sim D_{c_i^-}}}
\left[-\mathop{log}\left(\frac{1}{1+\sum_{i=1}^{k}e^{\left(\tiny{h^T \left(h_i^--h^+ \right)}\right)}}\right)\right] \right] \\
&= \mathop{\mathbb{E}}\limits_{\tiny{c, \{c_i^-\}_{i=1}^{k} \atop \sim C}}
\left[\mathop{\mathbb{E}}\limits_{\tiny{(x, x^+) \sim D_{c} \atop x_i^- \sim D_{c_i^-}}}
\left[l \left( h^T(h^+-h_i^-) \right)\right]\right]
\end{aligned}
$$

$c, \{c_i^-\}_{i=1}^{k} \sim C$ is the sampling process from a label distribution $C$. We take $k+1$ labels from this process, and we aim to separate label $c$ from other labels in $\{c_i^-\}_{i=1}^{k}$ to demonstrate discriminative ability of our method. $(x, x^+) \sim D_{c}$ denotes two observation perspectives of a same peptide sampled from $D_c$; $x_i^-\sim D_{c_i^-}$ denotes $x_i^-$ is negative sample of $(x,x^+)$  from another peptide with downstream label $c_i^-$.

 $l$ is a convex function w.r.t $x_i$, which can be expressed as:

$$
l(\{x_i\}_{i=1}^{n}) = -log(\frac{1}{1+\sum_{i=1}^n e^{-x_i}})
$$

According to Jensen’s Inequality, we have:

$$
\begin{aligned}
\mathcal{L}_{con} &= \mathop{\mathbb{E}}\limits_{\tiny{c,\{c_i^-\}_{i=1}^{k} \atop \sim C}}
\left[\mathop{\mathbb{E}}\limits_{\tiny{(x, x^+) \sim D_{c} \atop x_i^- \sim D_{c_i^-}}}
\left[l \left( h^T(h^+-h_i^-) \right)\right]\right] \\
&\geq \mathop{\mathbb{E}}\limits_{\tiny{c, \{c_i^-\}_{i=1}^{k} \atop \sim C,x\sim D_{c}}}
\left[l \left( \mathop{\mathbb{E}}\limits_{\tiny{x^+ \sim D_{c} \atop x_i^- \sim D_{c_i^-}}}\left[h^T(h^+-h_i^-) \right] \right)\right] \\
& = \mathop{\mathbb{E}}\limits_{\tiny{c^+, c^- \atop \sim C}}
\mathop{\mathbb{E}}\limits_{\tiny{x \sim D_{c}}} \left[ l \left(h^T(\mu_{c^+}-\mu_{c^-})\right) \right] \\
&  =
\lambda\mathop{\mathbb{E}}\limits_{\tiny{c, x}} \left[ l \left( \{g(h)_c - g(h)_{c'} \}_{c\neq c'} \right) \right]
\end{aligned}
$$

here $\lambda=\mathbb{P}(c'\neq c|c)$, $g(h)_c =h^T\mu_c = [W^{\mu}h]_c$, $W^{\mu}\in \mathbb{R}^{C \times d}$ whose $c_{th}$ row is the mean of embedding representation with label $c$ (i.e. $[W^{\mu}]_c =\mu_c= \mathbb{E}_{x\sim D_c}[h]$). 
In term of neural networks, we take a fully connected layer with no bias as an example downstream task predictor with input of representation $h\in\mathbb{R}^{d}$. Ideally, $W$ should be optimized by training, while in this case we approximate $W$ as $W^{\mu}$ . By take inner product between $h$ (assume having label $c$) and each row of $W^{\mu}$, the output vector has a larger value in $c^{th}$ row, so it act as classifier $W$. 

Thus, by minimizing contrastive loss $\mathcal{L}_{con}$, the classifier $W^{\mu}$ can better discriminate $x$ with label $c$ from other label $c'$.
\end{proof}

\section{Theoretical Support for $\mathcal{L}_{con}$ from a mutual information perspective} \label{proof_theo2}

\begin{theorem} \label{theo2}
Minimizing contrastive learning loss $L_{con}$ can increase the mutual information between sequence representation and graph representation $I(h;h^+)$.
\end{theorem}

\begin{proof}
Through optimizing contrastive loss, according to Theorem \ref{theo1}, model can separate representations with different downstream task labels, it is naturally to assume that $e^{(h, h^{’})}$ is proportional to $\frac{\mathbb{P}(h'|h)}{\mathbb{P}(h')}$, $\mathbb{P}(h'|h)$ denotes the probability that give representation $h$, $h'$ has same label with $h$, the denominator $\mathbb{P}(h')$ ensures the permutation invariant under $h'$ and $h$, thus we can reformulate contrastive learning loss:
$$
\begin{aligned}
L_{contrastive} &= \mathop{\mathbb{E}}\limits_{\tiny{c, \{c_i^{-}\}_{i=1}^{k} \atop \sim C}}
\mathop{\mathbb{E}}\limits_{\tiny{(x, x^+) \sim D_{c} \atop x_i^{-} \sim D_{c_i^{-}}}}
\left[-\mathop{log}\left(\frac{e^{\left(\tiny{h^{T}h^{+}}\right)}}{e^{\left(\tiny{h^{T}h^{+}}\right)}+\sum_{i=1}^{k}e^{\left(\tiny{h^{T}h_i^{-}}\right)}}\right)\right] \\
&= 
\mathop{\mathbb{E}}
\left[-\mathop{log}\left(\frac{\frac{\mathbb{P}(h^{+}|h)}{\mathbb{P}(h^{+})}}{\frac{\mathbb{P}(h^{+}|h)}{\mathbb{P}(h^{+})}+\sum_{i=1}^{k}\frac{\mathbb{P}(h^{-}|h)}{\mathbb{P}(h^{-})}}\right)\right] \\
&\approx \mathop{\mathbb{E}}
\left[\mathop{log}\left(1+\frac{\mathbb{P}(h^{+})}{\mathbb{P}(h^{+}|h)}k\mathbb{E}\left[\frac{\mathbb{P}(h^{-}_i|h)}{\mathbb{P}(h^{-}_i)}\right]\right)\right] \\
&\overset{1}{=}\mathop{\mathbb{E}}
\left[\mathop{log}\left(1+\frac{\mathbb{P}(h^{+})}{\mathbb{P}(h^{+}|h)}k\right)\right] \\
&\geq \mathop{\mathbb{E}}
\left[\mathop{log}\left(\frac{\mathbb{P}(h^{+})}{\mathbb{P}(h^{+}|h)}k\right)\right] \\
&= -I(h^{+};h)+log(K)
\end{aligned}
$$
which is based on the fact that $h_i^{-}$ is independently sampled, i.e. $\mathbb{P}(h_i^{-}|h)=\mathbb{P}(h_i^{-})$.

\end{proof}

\section{Pseudocodes of Training and Testing Phases of RepCon (Q1)} \label{appen3}

The pseudocodes of both training and testing phases of RepCon are provided in Algorithm \ref{alg1}.

\begin{algorithm}[h]
\caption{Training and Testing of RepCon}
\label{alg1}
\begin{algorithmic}[h]
\REQUIRE
Train and test data, hyper-parameter $\lambda$ (Eq.~\ref{L_train}), learning rate $\gamma$

\STATE \textit{-Training Phase-}

\STATE Initialize learnable parameters randomly

\WHILE{not converge}
\STATE Sample minibatch $B$ from train data
\FOR{each sample $(x_{seq},y) \in B$}
\STATE Transform $x_{seq}$ from FASTA to molecular graph $x_{graph}$
\STATE Compute supervised loss by Eq.~\ref{L_pred}
\STATE Compute unsupervised InfoNCE loss by Eq.~\ref{L_infonce}
\ENDFOR
\STATE Calculate the gradients on learnable parameters by backpropagation
\STATE Update learnable parameters with learning rate $\gamma$
\ENDWHILE

\STATE \textit{-Testing Phase-}
\STATE Load the sequential encoder and predictor
\FOR{each sample $x_{seq}$ in test data}
\STATE Inference the label by Eq.~\ref{inference}
\ENDFOR

\end{algorithmic}
\end{algorithm}

\section{Experiment-related}

\subsection{Details of Datasets} \label{appen4}

The statistics of datasets AP, RT, AMP and PepDB are shown in Table \ref{statistics}.
In addition to the parameters addressed in the table, the four datasets uniformly contain 20 natural amino acids.
In dataset PepDB, the classes 0-2 stand for: class 0 antimicrobial peptides; class 1 peptide hormones; class 2 toxins and venom peptides.
In dataset AMP, the class 1 stands for the antimicrobial peptide, and the class 0 stands for the non-antimicrobial peptide.

\begin{table}[h]
\caption{Statistics of datasets.}
\centering
\begin{tabular}{lcccc}
\hline
Dataset & Task           & Num. of Samples   & Classes & Max. Length \\ \hline
AP      & Regression     & 62,159  & -       & 10     \\
RT      & Regression     & 121,215 & -       & 50     \\
AMP     & Classification & 9,321   & 2       & 50     \\
PepDB   & Classification & 7,016   & 3       & 50     \\ \hline
\end{tabular}
\label{statistics}
\end{table}

\subsection{Implementation of Co-modeling Baselines (Q2)} \label{appen5}

Co-modeling Baselines include weighted sum (WS), concatenation (Concat), cross attention (CA), and compact bilinear pooling (CBP), which are fusion methods pluged in our proposed co-modeling framework.
The inputs of fusion block are the output representations, $h_{seq}$ and $h_{graph}$, from the sequential and graphical encoders.
Denote the output representation from fusion block is $h_{co-rep}$, the implementation of baseline fusion methods are detailed as follows.

\noindent \textbf{Weighted Sum (WS)} simply merges the input representations, denoted as:
\begin{equation}
    h_{co-rep}=\delta h_{seq}+(1-\delta) h_{graph},
    \nonumber
\end{equation}
where $\delta$ is a balance hyperparameter. In the experiments, it is set as 0.5.

\noindent \textbf{Concatenation (Concat)} aligns the input representations and results in a longer dimensional representation, denoted as:
\begin{equation}
    h_{co-rep}=[h_{seq}, h_{graph}].
    \nonumber
\end{equation}

\noindent \textbf{Cross Attention (CA)} an extension of the self-attention mechanism used for aligning two representations. The mathematical expression can be denoted as:
\begin{equation}
    h_{co-rep}=softmax(\frac{h_{graph}\; {h_{graph}}^T}{\sqrt{d_{h_{graph}}}})h_{seq},
    \nonumber
\end{equation}
where $h_{graph}$ is considered as the key and query, $h_{seq}$ is considered as the value, and $d_{h_{graph}}$ is the dimension of $h_{graph}$.

\noindent \textbf{Compact Bilinear Pooling (CBP)} combines features from different sources or modalities, aiming to capture rich interactions. The formulation of CBP is denoted as:
\begin{equation}
h_{co-rep} = \mathcal{F}^{-1}(\mathcal{F}(h_{seq}) \odot \mathcal{F}(h_{graph})),
    \nonumber
\end{equation}
where $\mathcal{F}(\cdot)$ represents the Fourier Transform operator, $\mathcal{F}^{-1}(\cdot)$ represents the inverse Fourier Transform operator, and $\odot$ represents element-wise multiplication.

\subsection{Evaluations for Choosing Backbones (Q1)} \label{appen6}

\begin{table}[]
\caption{Evaluations on sequential and graphical backbone models. Each backbone is connected with a uniformed MLP predictor. Each experimental result is conducted under random seed 5. The result of the best performed model in each group is bolded. As the result of this experiment, we select Transformer and GraphSAGE as the backbones of the sequential encoder and graphical encoder for evaluating co-modeling methods.}
\centering
\begin{tabular}{llcccc}
\hline
& \multirow{2}{*}{Model} & \multicolumn{2}{c}{Regression (MAE)\;\;\;\;\;} & \multicolumn{2}{c}{Classification (Acc.)} \\
&                        & AP                & RT               & AMP                & PepDB                \\ \hline
\multirow{4}{*}{Sequential Backbone}                 & RNN                    & 4.52E-2                  & 1.95                 & 82.5\%                   & 93.3\%                     \\
& LSTM                   & 4.28E-2                  & 1.79                 & 83.5\%                   & 95.4\%                     \\
& Bi-LSTM                & 4.25E-2                  & \textbf{1.54}                 & 82.8\%                   & 95.7\%                     \\
& \textbf{Transformer}            & \textbf{3.81E-2}                  & 1.57                 & \textbf{84.7\%}                   & \textbf{96.1\%}                     \\ \hline
\multirow{4}{*}{Graphical Backbone} & GCN                    & 4.11E-2                  & 3.34                 & 84.2\%                   & 93.4\%                     \\
& GAT                    & 4.09E-2                  & 2.99                 & \textbf{84.7\%}                   & 94.5\%                     \\
& PNA                    &  3.93E-2                 & 2.72                 & 83.5\%                   & 93.7\%                     \\
& \textbf{GraphSAGE}              & \textbf{3.89E-2}                  & \textbf{2.57}                 & \textbf{84.7\%}                   & \textbf{95.2\%}                     \\ \hline
\end{tabular}
\label{backbones}
\end{table}

In this section, we show experimental results for sequential  and graphical backbone models on all datasets involved in Section \ref{exp} under the same experimental settings.
Sequential backbone models include RNN, LSTM, Bi-LSTM and Transformer.
Graphical backbone models include GCN, GAT, PNA, and GraphSAGE with proper hyperparameter settings.

The results of above mentioned backbones are shown in Table \ref{backbones}.
Among the sequential backbones, Transformer ranks the first on three datasets.
On dataset RT, Transformer is only 0.03 of MAE behind the best-performed model Bi-LSTM.
Among the graphical backbones, GraphSAGE ranks the first on all tested datasets, which means it is consistently outperforming other baselines.
Therefore, Transformer and GraphSAGE are chosen as the sequential and graphical encoders of the co-modeling framework. (Q1)
However, to further improve the performance, the encoders can be replaced by any state-of-the-art feature extraction module for the specific downstream task.

\subsection{Amino acid-wise attribution and MAE of RepCon (Q3\&Q4)} \label{appen7}

In this section, we provide the evaluations on RepCon's attribution and MAE in a statistical amino acid level, shown in Fig.~\ref{fig_sup1}.
The left subfigure in Fig.~\ref{fig_sup1} shows the amino acid-wise average attribution value, same as the left subfigure in Fig.~\ref{fig1}, for RepCon and backbones.
The right subfigure in Fig.~\ref{fig_sup1} shows the amino acid-wise MAE, i.e., the average absolute error for peptide samples containing at least one target amino acid.

From the left subfigure, the attribution performance of RepCon demonstrates the extraction of information from Transformer and GraphSAGE to RepCon.
From the right subfigure, the improvement in MAE proves that the RepCon's evolution in attribution is in a beneficial direction.
In addition, comparing Transformer and GraphSAGE (as they have competitive performance), amino acids with large differences in attribution in the left also show large differences in MAE, e.g., amino acids A, C, D, N, and Q. In contrast, amino acids with similar attribution between Transformer and GraphSAGE, e.g., F, W, and Y, show similar MAE.

\begin{figure}[h] 
    \centering
    \subfloat{\includegraphics[width=0.49\linewidth]{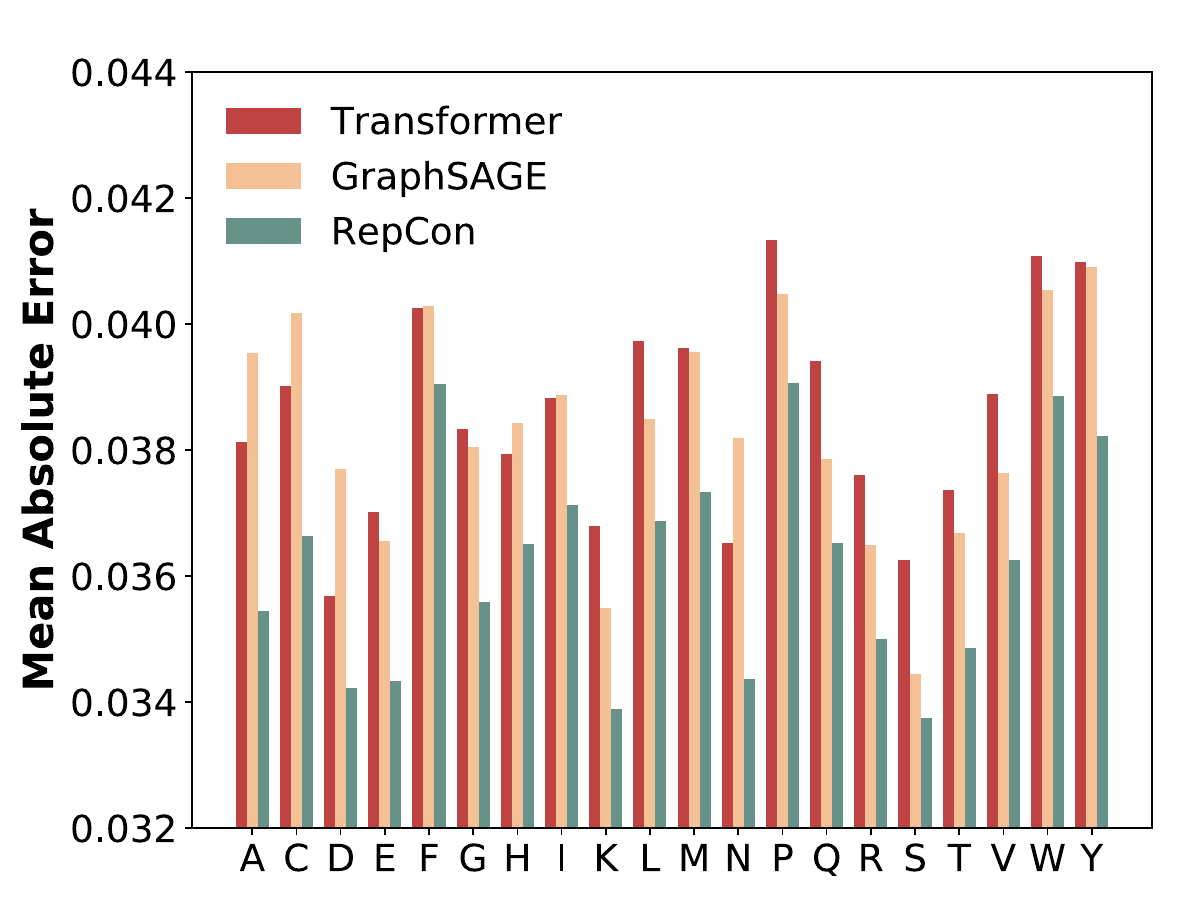}}
    \subfloat{\includegraphics[width=0.49\linewidth]{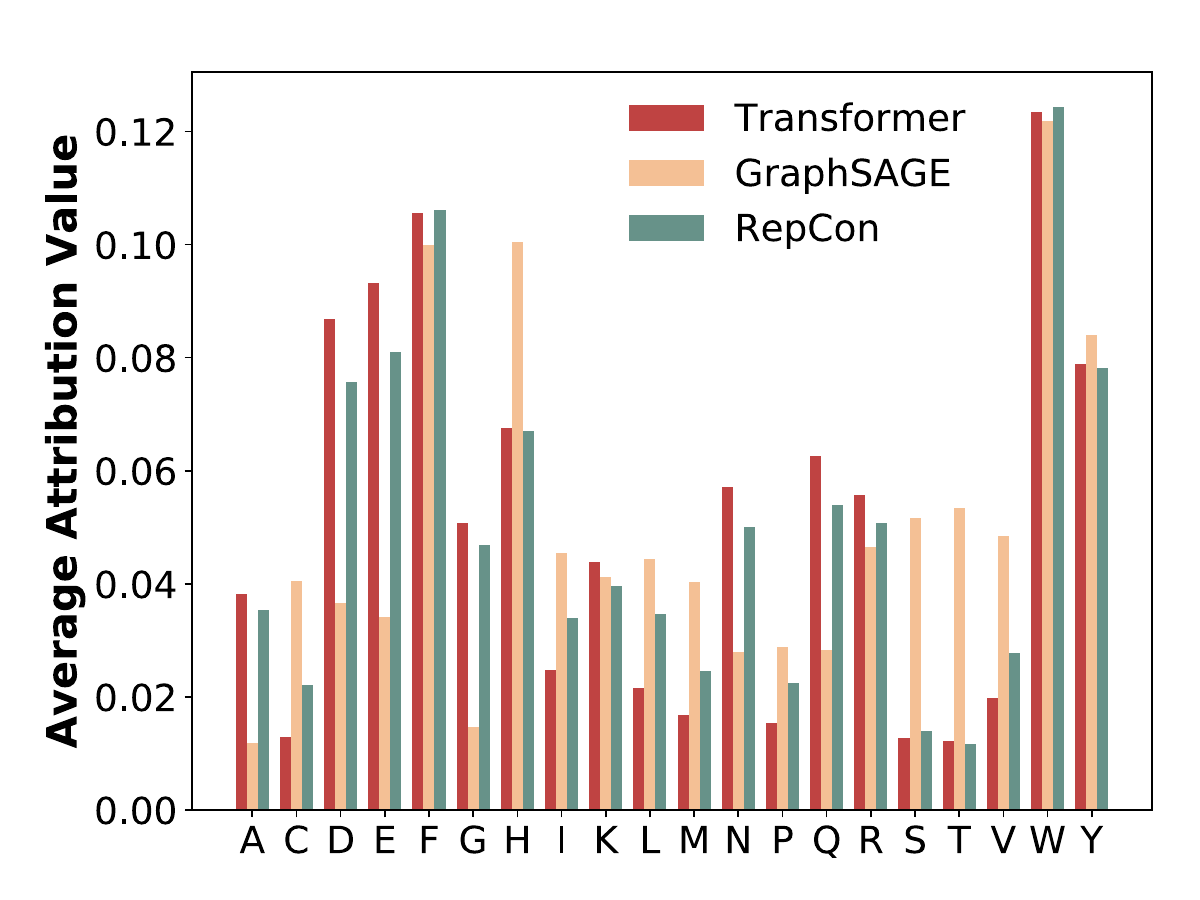}}
    \caption{Amino acid-wise attribution and MAE of RepCon compared with backbone Transformer and backbone GraphSAGE on AP regression task.}
    \label{fig_sup1}
\end{figure}

\subsection{RepCon's Loss Curves (Q3)} \label{appen8}

\begin{figure}[h] 
    \centering
    \subfloat{\includegraphics[width=0.45\linewidth]{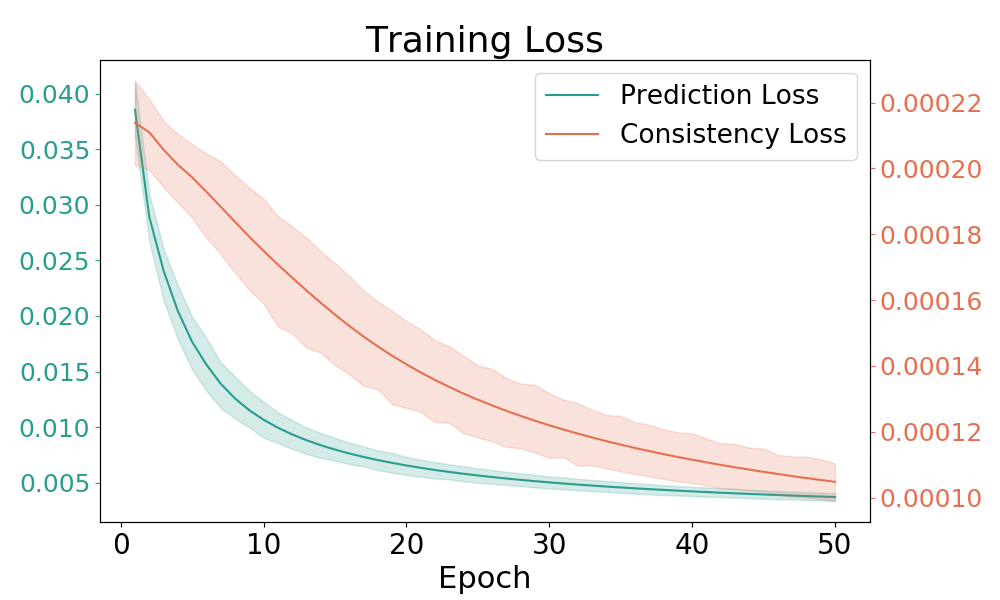}}
    \subfloat{\includegraphics[width=0.45\linewidth]{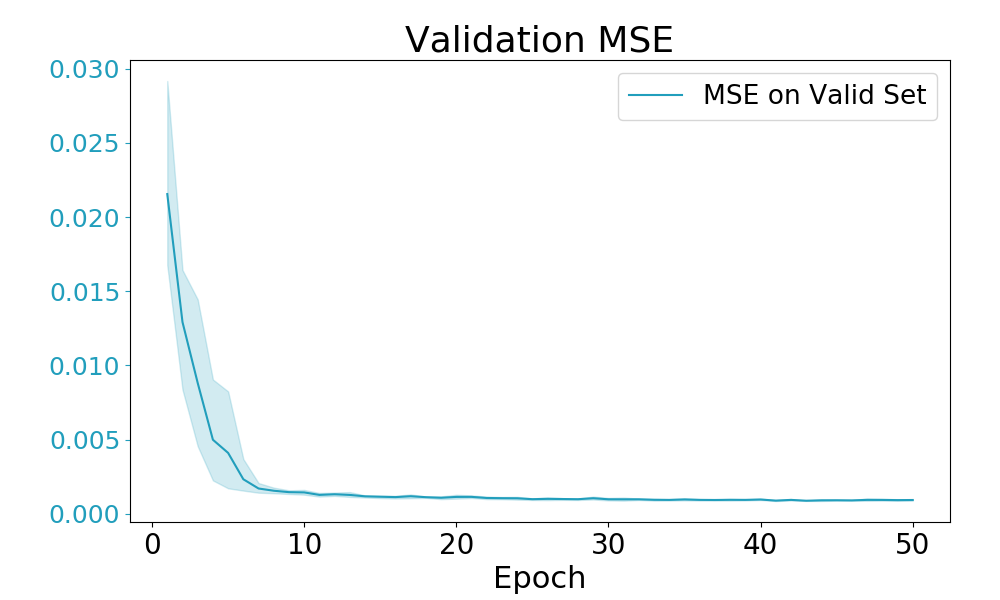}}
    \caption{(Left) loss curves for prediction loss $\mathcal{L}_{pred}$ and consistency loss $\mathcal{L}_{con}$; (Right) MSE performance on validation set. }
    \label{supfig2}
\end{figure}

In Fig.~\ref{supfig2}, the left subfigure shows the convergence trend of terms prediction loss $\mathcal{L}_{pred}$ and consistency loss $\mathcal{L}_{con}$ in the training loss consistency loss $\mathcal{L}_{train}$ on dataset AP, and the right subfigure shows the MSE of our proposed RepCon on the validation set.
The horizontal axis represents the epoch and the vertical axis represents the value of loss or MSE.
In the left subfigure, the vertical axis for line of prediction loss is indicated by the green font on the left, and the vertical axis for line of consistency loss is indicated by the red font on the right.
From Fig.~\ref{supfig2}, we observe that the training loss is steadily converging.
The validation performance is improving with training and no overfitting is observed.

\subsection{Sensitivity on Hyperparameter (Q1)} \label{appen9}

This section provide the sensitivity of hyperparameter $\lambda$, which is used to balance the loss terms in Eq.~\ref{L_train}. 
The experimental results for $\lambda$'s sensitivity are shown in Fig.~\ref{sup_sensitivity}.
We have conducted one experiment for each $\lambda$ value on all four datasets involved in the experiment section.
From the results, we observe that the proper value for $\lambda$ in classification tasks is in the interval of [0.01, 0.1].
Regarding to regression tasks, the proper $\lambda$ value is in the interval of [1e-5, 1e-4].
From the view of ablation study, the existence of the contrastive loss term benefits the training of RepCon, as the general performance of most experiments with $\lambda>0$ surpasses the backbone models.
We adopt the proper value for $\lambda$ in each downstream datasets.

\begin{figure}[h] 
    \centering
    \subfloat{\includegraphics[width=0.45\linewidth]{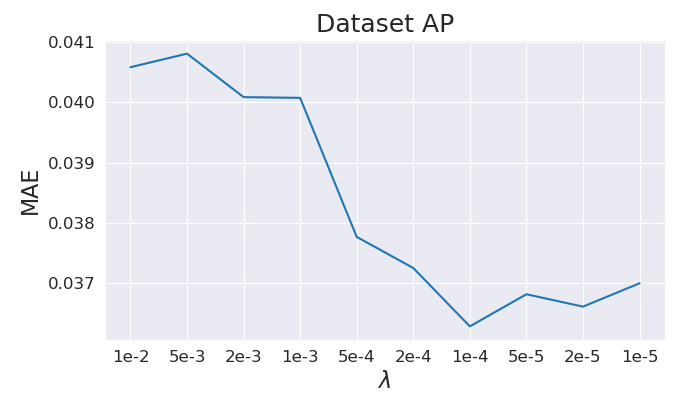}}
    \subfloat{\includegraphics[width=0.45\linewidth]{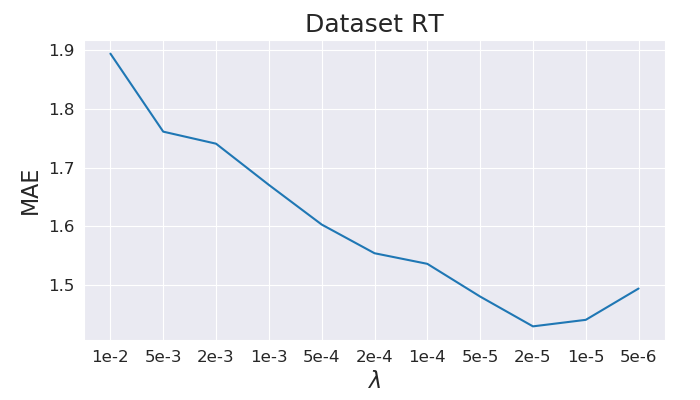}}
    \\
    \subfloat{\includegraphics[width=0.45\linewidth]{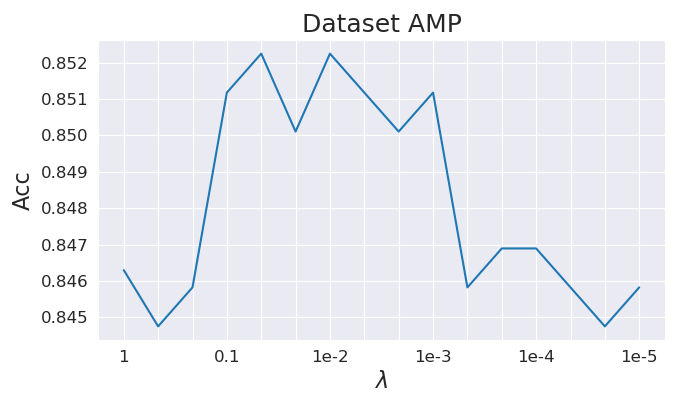}}
    \subfloat{\includegraphics[width=0.45\linewidth]{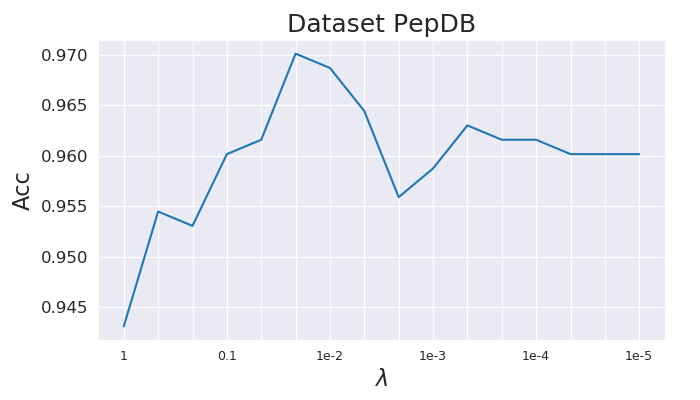}}
    \caption{Sensitivity experiments for hyperparameter $\lambda$. Due to the randomness of model initialization and training, it is reasonable that MAE corresponding to continuous $\lambda$ fluctuates. This experiment is mainly used for the selection of $\lambda$.}
    \label{sup_sensitivity}
\end{figure}

\subsection{Attribution Comparison for Transformer and GraphSAGE in Cases}

\begin{figure}[htbp]
  \centering
  \subfloat{
    \includegraphics[width=0.249\textwidth]{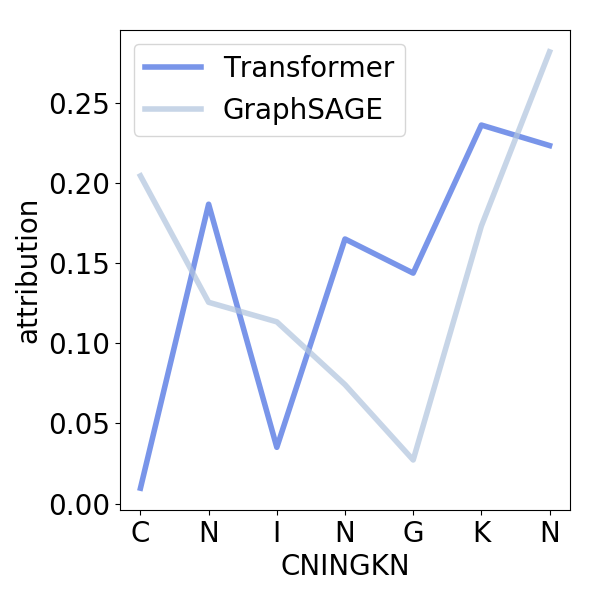}
  }
  \subfloat{
    \includegraphics[width=0.249\textwidth]{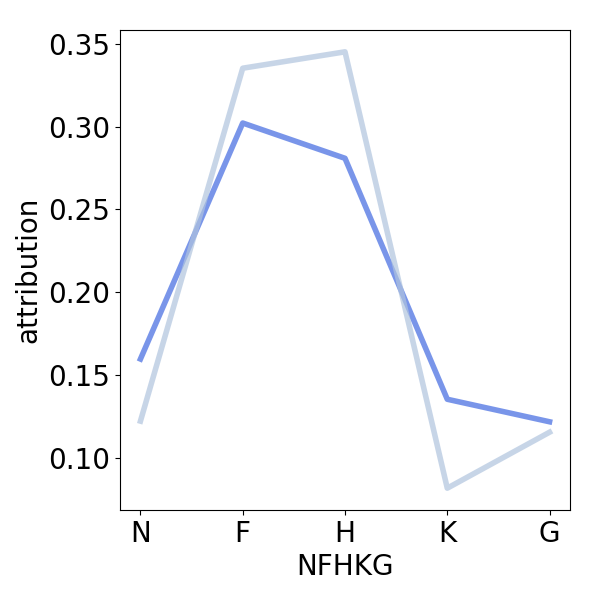}
  }
  \subfloat{
    \includegraphics[width=0.249\textwidth]{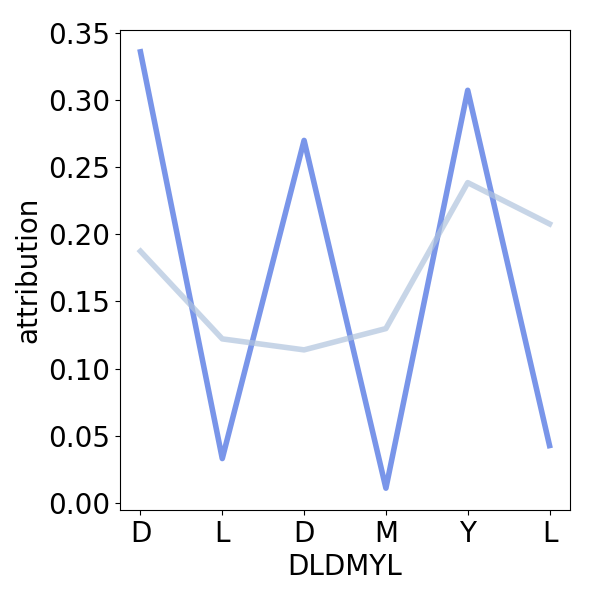}
  }
  \subfloat{
    \includegraphics[width=0.249\textwidth]{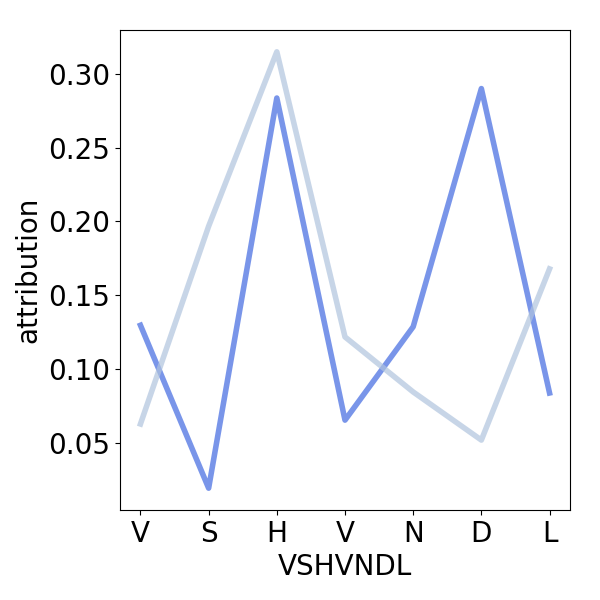}
  }
  \\
  \subfloat{
    \includegraphics[width=0.249\textwidth]{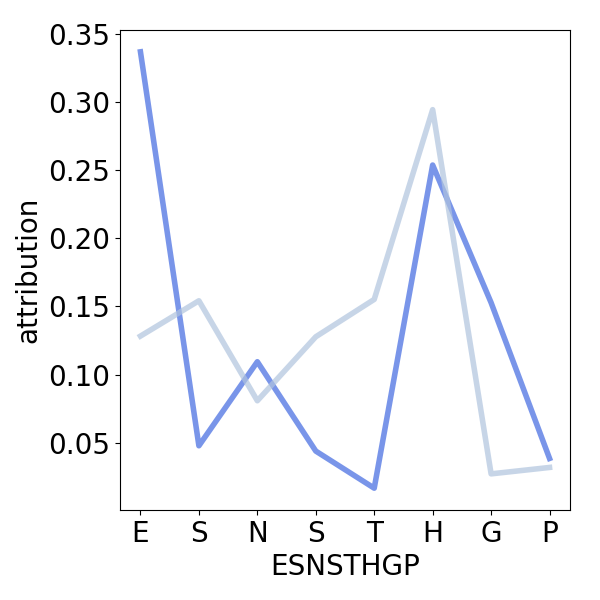}
  }
  \subfloat{
    \includegraphics[width=0.249\textwidth]{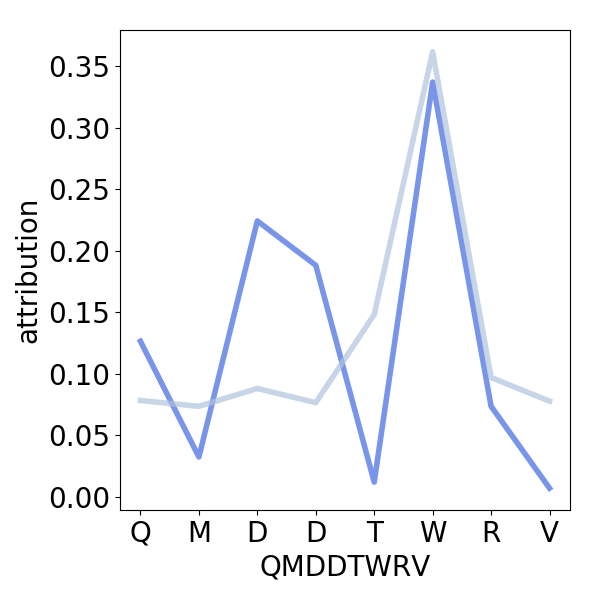}
  }
  \subfloat{
    \includegraphics[width=0.249\textwidth]{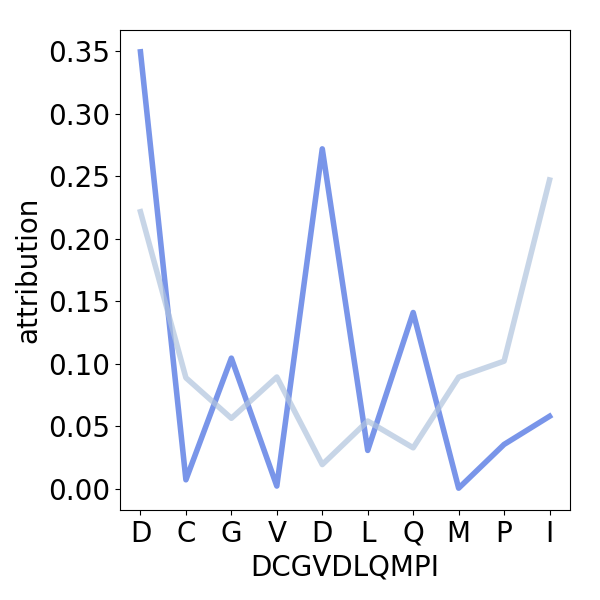}
  }
  \subfloat{
    \includegraphics[width=0.249\textwidth]{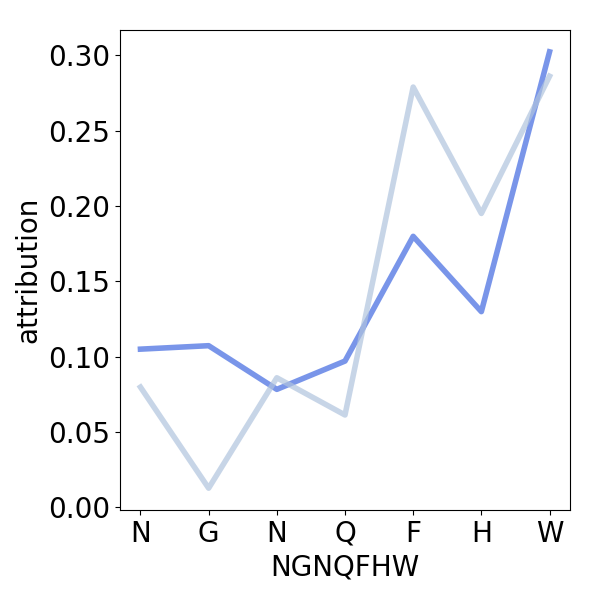}
  }
  \\
  \subfloat{
    \includegraphics[width=0.249\textwidth]{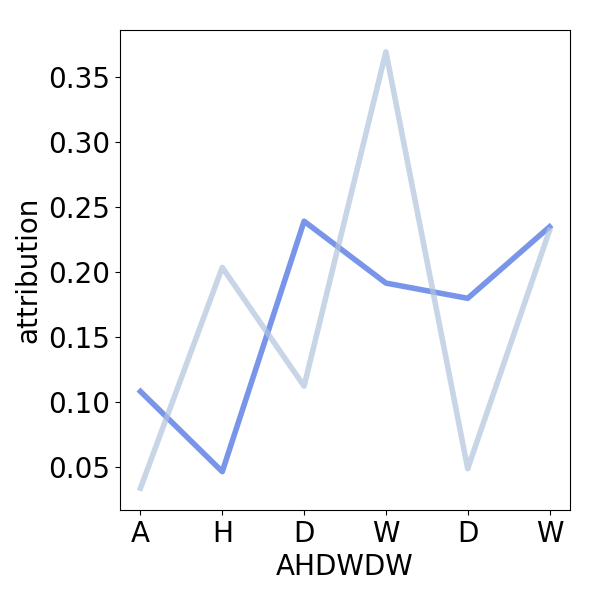}
  }
  \subfloat{
    \includegraphics[width=0.249\textwidth]{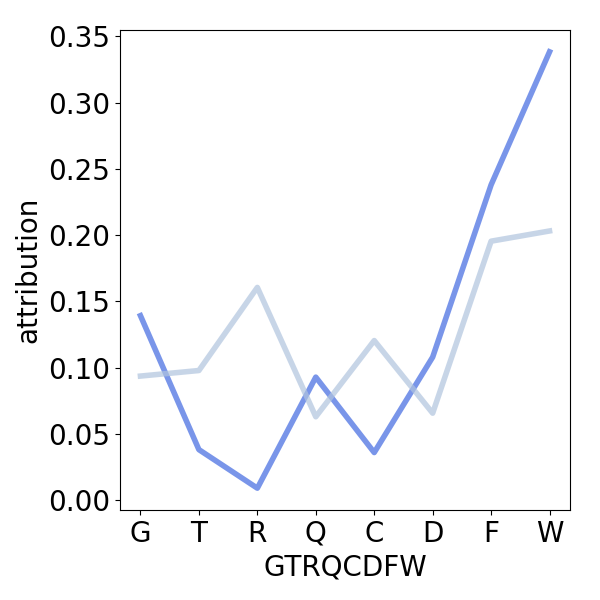}
  }
  \subfloat{
    \includegraphics[width=0.249\textwidth]{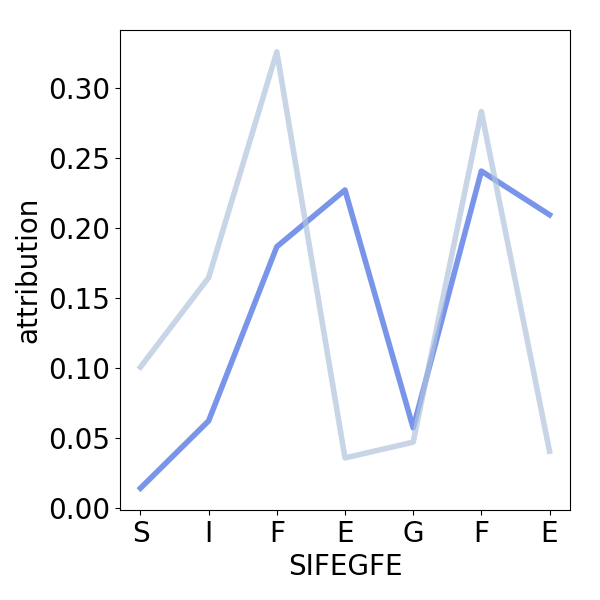}
  }
  \subfloat{
    \includegraphics[width=0.249\textwidth]{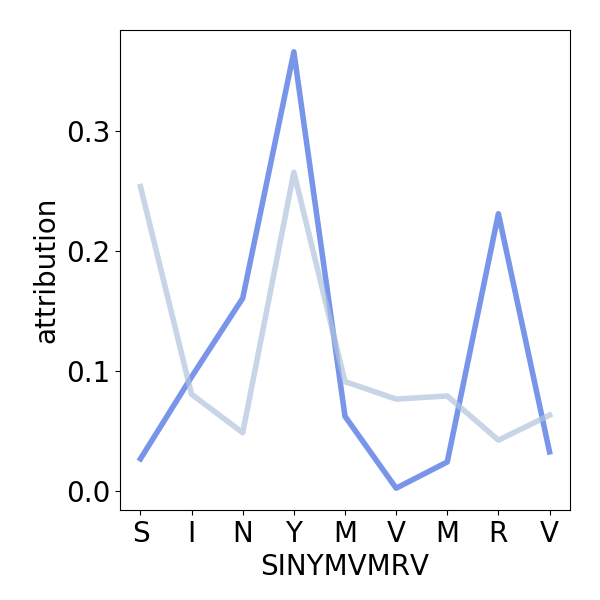}
  }
  \\
  \subfloat{
    \includegraphics[width=0.249\textwidth]{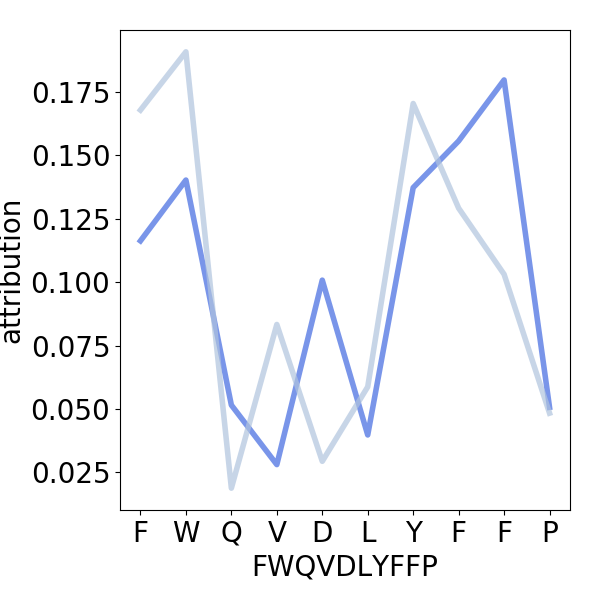}
  }
  \subfloat{
    \includegraphics[width=0.249\textwidth]{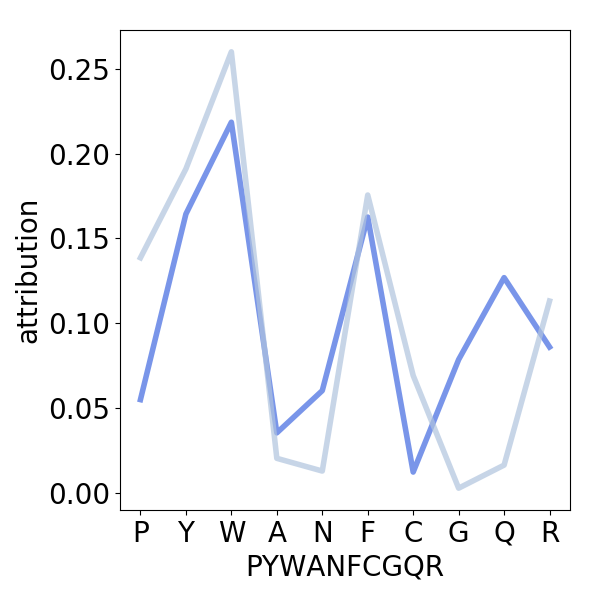}
  }
  \subfloat{
    \includegraphics[width=0.249\textwidth]{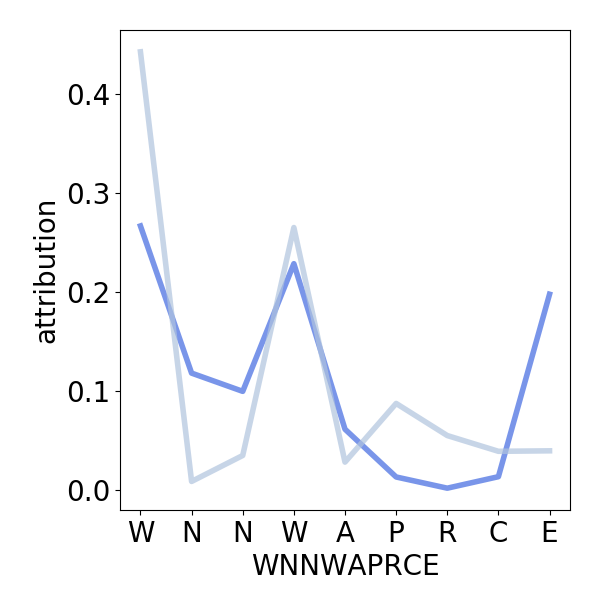}
  }
  \subfloat{
    \includegraphics[width=0.249\textwidth]{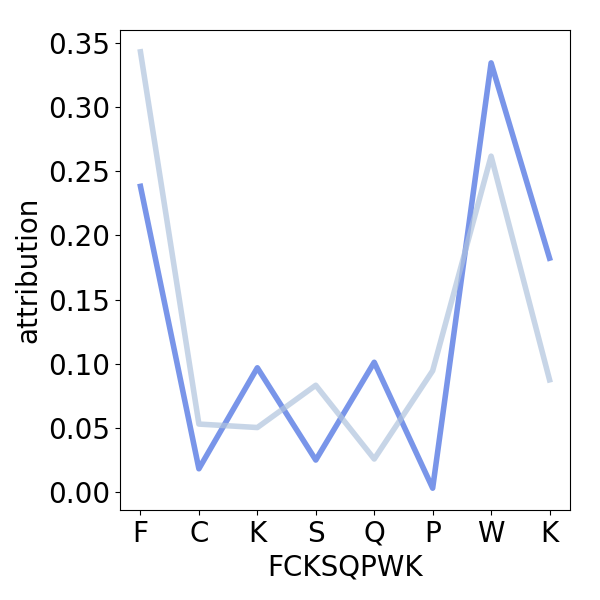}
  }
  \caption{Comparison of Transformer and GraphSAGE attribution distributions on random selected peptide examples on the dataset AP. The horizontal axis represents the amino acid components in the peptides and the vertical axis represents the normalized attribution values. The attribution distributions for Transformer are plotted in royal blue, and those for GraphSAGE are plotted in light steel blue.}
  \label{supfig:cases}
\end{figure}

This is a complementary experiment on attribution, in order to show in more detail the differences between Transformer and GraphSAGE.
Fig.~\ref{supfig:cases} shows the comparison of Transformer and GraphSAGE attribution distributions on random selected peptide examples on the dataset AP.
We observe high consistency in the attribution distributions in some of the peptide samples, such as NFHKG (row-1, column-2), NGNQFHW (row-2, column-4), and FCKSQPWK (row-4, column-4).
In these peptides, the model's determination of important amino acids in peptides is similar.
However, in many of the samples, we observed large differences in the model's determination of the important components.
For example, in ESNSTHGP (row-2, column-1), AHDWDW (row-3, column-1), and SIFEGFE (row-3, column-3), Transformer has a high attribution for the negatively charged amino acids D amd E (which in the AHDWDW samples even exceeds that of the W that contains an aromatic ring).
To summarize, for the majority of examples, there are similar attribution distributions for most components in some peptides, yet there are dissimilar attribution values for a few peptide compositions, and these dissimilarities tend to be concentrated on a few specific amino acids.

\end{document}